\documentclass{article}

\usepackage{microtype}
\usepackage{graphicx}
\usepackage{subfigure}
\usepackage{booktabs} % for professional tables
\usepackage{hyperref}

\usepackage[noend]{algorithmic}
\usepackage[accepted]{icml2021}

\usepackage{amssymb}
\usepackage{epstopdf}
\usepackage{amsmath}
\usepackage{amsthm}
\usepackage{url}
\usepackage{color}
\usepackage{epstopdf}
\usepackage{bm}
\usepackage{dsfont}
\usepackage[ISO]{diffcoeff}
\usepackage{optidef}
\usepackage{xspace}

\renewcommand{\cite}{\citep}
\newtheorem{thm}{Theorem}
\newtheorem{prop}[thm]{Proposition}

\newtheorem{cor}[thm]{Corollary}
\newtheorem{deff}[thm]{Definition}

\newcommand{\cH}{\mathcal{H}}

\newcommand{\cN}{\mathcal{N}}
\newcommand{\cO}{\mathcal{O}}
\newcommand{\cS}{\mathcal{S}}
\newcommand{\cT}{\mathcal{T}}

\newcommand{\cX}{\mathcal{X}}

\newcommand{\bbR}{\mathbb{R}}
\newcommand{\bbN}{\mathbb{N}}
\newcommand{\bbZ}{\mathbb{Z}}
\newcommand{\bbE}{\mathbb{E}}

\newcommand{\bx}{\mathbf{x}}

\newcommand{\bp}{\mathbf{p}}
\newcommand{\bu}{\mathbf{u}}

\newcommand{\bone}{\mathbf{1}}
\newcommand{\bbone}{\mathds{1}}
\newcommand{\conv}[1]{\mathrm{conv}(#1)}
\newcommand{\convzero}[1]{\mathrm{conv}_{0}(#1)}

\newcommand{\ie}{\textit{i.e.}}
\newcommand{\eg}{\textit{e.g.}}

\newcommand{\pp}{\mathcal{PP}\xspace}
\newcommand{\dpp}{\partial\mathcal{PP}\xspace}
\newcommand{\mpp}{\mathcal{MPP}\xspace}
\newcommand{\cat}{\mathrm{Cat}}
\newcommand{\bpi}{\bm{\pi}}
\newcommand{\blambda}{\bm{\lambda}}

\icmltitlerunning{A Differentiable Point Process with Its Application to Spiking Neural Networks}
\begin{document}

\twocolumn[
\icmltitle{A Differentiable Point Process with Its Application to Spiking Neural Networks}



\icmlsetsymbol{equal}{*}

\begin{icmlauthorlist}
\icmlauthor{Hiroshi Kajino}{ibm}
\end{icmlauthorlist}

\icmlaffiliation{ibm}{IBM Research - Tokyo, Tokyo, Japan}

\icmlcorrespondingauthor{Hiroshi Kajino}{kajino@jp.ibm.com}

\icmlkeywords{Point process, Spiking neural networks}

\vskip 0.3in
]



\printAffiliationsAndNotice{}  % leave blank if no need to mention equal contribution

\begin{abstract}
 This paper is concerned about a learning algorithm for a probabilistic model of spiking neural networks~(SNNs).
 \citeauthor{10.3389/fncom.2014.00038}~(\citeyear{10.3389/fncom.2014.00038}) proposed a stochastic variational inference algorithm to train SNNs with hidden neurons.
 The algorithm updates the variational distribution using the score function gradient estimator, whose high variance often impedes the whole learning algorithm.
 This paper presents an alternative gradient estimator for SNNs based on the path-wise gradient estimator.
 The main technical difficulty is a lack of a general method to differentiate a realization of an arbitrary point process,
 which is necessary to derive the path-wise gradient estimator.
 We develop a differentiable point process, which is the technical highlight of this paper, and apply it to derive the path-wise gradient estimator for SNNs.
 We investigate the effectiveness of our gradient estimator through numerical simulation.
\end{abstract}

\section{Introduction}
A spiking neural network~(SNN) is an artificial neural network~(ANN) where neurons communicate with each other using \emph{spikes} rather than \emph{real values} as the conventional ANNs do.
The conventional ANN is a special case of SNN where information is encoded into the firing rate of neurons~(which we call the rate coding) and the rate serves as the communication currency.
This specification facilitates developing learning algorithms for ANNs, leading to the recent great success of deep neural networks.
On the other hand, in the community of neuroscience, experimental evidence on biological neurons indicates that the rate coding alone cannot explain the whole brain activity~\cite{Bothe2004} and more precise modeling of neural coding is anticipated.
Since there still exist performance gaps between the rate-based ANNs and biological neural networks~(\ie, brains) in terms of inference capability and energy efficiency,
this raises the following question: how much of the current performance gaps can be attributed to this difference on neural coding?
This open problem motivates us to study SNNs.

One of the major obstacles towards answering it is a lack of practical learning algorithms for SNNs,
which discourages us from empirical investigation.
While there exist a number of attempts to develop learning algorithms, most of them have more or less limited applicability.
We consider a practical learning algorithm should at least be (i)~theoretically grounded, (ii)~empirically confirmed to work well, and (iii)~easy to simulate~(fewer hyperparameters, less computation time, etc.)\footnote{Biological plausibility is of another great interest, but is not mentioned because it is not mandatory for engineering purposes.}.
For example, theoretical aspects of the algorithms based on spike-timing-dependent plasticity~(Chapter~19~\cite{gerstner2014}) are not well understood.
For another example, simulating learning algorithms for continuous-time deterministic SNNs requires the step-size parameter of time-axis discretization when the dynamics of a neuron is described by differential equations~(\eg, \cite{NIPS2018_7417}). The step-size parameter brings about the trade-off between the simulation quality and computation time, which makes the simulation more intricate.
These examples illustrate that even major approaches do not satisfy all the requirements above, and therefore, there still exists much room for improvement.

Among a number of approaches, we employ as a foundation a probabilistic formulation of SNNs~\cite{doi:10.1162/neco.2006.18.6.1318},
which models spike trains~(temporal sequence of spikes emitted from neurons) as a realization of a multivariate point process.
It is easier for us to start from it than others because it already satisfies requirements (i) and (iii), which are more intrinsic properties than requirement (ii).
In fact, learning algorithms are formalized by maximum likelihood estimation, and
its exact simulation has no trade-off hyperparameter as will be explained in Section~\ref{sec:sampling-from-point}.
Therefore, the remaining concern is its empirical performance.

One of the state-of-the-art learning algorithms for probabilistic SNNs is the work by \citeauthor{10.3389/fncom.2014.00038}~(\citeyear{10.3389/fncom.2014.00038}).
The authors propose a stochastic variational inference algorithm for SNNs with hidden neurons.
Since spike trains of hidden neurons are unobservable and it is intractable to compute the marginal likelihood,
an evidence lowerbound~(ELBO) is instead used as the objective function~(Section~\ref{sec:learning-algorithms}).
The key factor for optimizing ELBO is the way we estimate the gradient of ELBO.
%It is known that the performance of the learned model greatly depends on how we estimate the gradient of ELBO.
The authors employed the score function gradient estimator, also known as the REINFORCE estimator, which is widely applicable but is often reported to suffer from its high variance.

Our main idea is to substitute a path-wise gradient estimator for the score function gradient estimator.
The path-wise gradient estimator tends to have lower variance than the score function gradient estimator~\cite{mohamed2019monte},
but it is not widely applicable~(and is not applicable to SNNs) because it requires a sample from the variational distribution to be differentiable.
%To this end, it is necessary to develop a point process whose realization is differentiable with respect to the model parameters.
Our contribution is that we develop a differentiable point process~(Section~\ref{sec:diff-point-proc}) and apply it to derive the path-wise gradient estimator for SNNs~(Section~\ref{sec:grad-with-resp-phi}). %, thus allowing us to construct a path-wise gradient estimator for SNNs.
\if0
$\dpp$ is a point process whose realization is differentiable with respect to the model parameter,
\fi

We empirically investigate the effectiveness of the proposed learning algorithm in Section~\ref{sec:empirical-study}.
We will confirm that (i)~the proposed gradient estimator has lower variance than the existing one and (ii)~this lower variance contributes to improve the performance of the learning algorithm.
By comparing the performance of the proposed and existing ones, we obtain experimental results supporting these hypotheses.
Therefore, we conclude that our path-wise gradient estimator improves empirical performance of SNNs.

One of the limitations of our learning algorithm as compared to the existing algorithm~\cite{10.3389/fncom.2014.00038} is computation time.
Since our algorithm generates more hidden spikes than the existing one does, our algorithm requires more computation time.
We empirically examine the computational overhead of our algorithm against the existing one, and find that our algorithm requires 2.8 times more computation time than the existing one.

%\paragraph{Notation.}
\textbf{Notation.} Let $[N]=\{1,2,\dots,N\}$.
%A scalar and a scalar-valued function is denoted by an italic letter, \eg, $a$.
For any vector~$\bx$, its $d$-th element is represented by $x_d$.
$\begin{bmatrix}x_d\end{bmatrix}_{d\in[D]}$ denotes a $D$-dimensional vector whose $d$-th element is $x_d$.
For any vector~$\bx\in\bbR^D$ and scalar $c\in\bbR$, $\begin{bmatrix}\bx^\top & c\end{bmatrix}^\top$ denotes the $(D+1)$-dimensional vector concatenating $\bx$ and~$c$.
%A matrix and matrix-valued function is denoted by a bold italic letter, \eg, $\bm{A}$, and its $(i,j)$ element is denoted by $A_{ij}$.
%Let $\bone$
Let $\bbR_{\geq 0}=\{x \geq 0 \}$ and $\bbR_{> 0}=\{x > 0 \}$.
%Let $\delta(x)$ be the Dirac delta function, and $\delta_{i,j}$ be the Kronecker delta.
%
%For any integer $D\in\bbN$, let $\Delta^{D-1}$ be the probability simplex in the $D$-dimensional Euclidean space, \ie, $\Delta^{D-1}=\{\bx\in\bbR^{D} \mid \mathbf{0}\leq \bx \leq \bone, \bone^\top \bx = 1\}$, where the inequalities are element-wise.
Let $\bbone^D=\{\bone_d \}_{d\in[D]}$ be the set of $D$-dimensional one-hot vectors, where $\bone_d\in\{0,1\}^D$ is the one-hot vector whose $d$-th element is $1$ and the others are $0$.
For any set $A$, let $\conv{A}$ be its convex hull, let $\convzero{A}:=\conv{A\cup\{0\}}$.
%Note that $\Delta^{D-1}$ is the convex hull of $I^D$.
%
%Let $\mathrm{Bernoulli}(p)$ be the Bernoulli distribution with parameter $p\in[0,1]$. The random variable takes $1$ with probability $p$ and $0$ with probability $1-p$.
Let $\cat(\bp)$ be the categorical distribution with parameter $\bp\in\conv{\bbone^D}$,
whose random variable takes $\bone_d\in\bbone^D$ with probability $p_d$.
Let $U[a,b]$ denote the uniform distribution over $[a,b]$.
For any expectation operator $\bbE_p$, let $\hat{\bbE}_p$ be its Monte-Carlo approximation using an i.i.d. sample from~$p$.
%Letting $\bp=\begin{bmatrix}\bp_1 & \bp_2 \end{bmatrix}\in\conv{\bbone^D}$, the categorical distribution with parameter $\bp$ is also denoted by $\mathrm{Categorical}(\bp_1,\bp_2)$.

\section{Preliminaries}
This section introduces temporal point processes along with their parameter estimation and sampling methods.

\subsection{Point Processes}

A point process~\cite{daley2003} is a probabilistic model of an event collection. It is called a \emph{temporal point process} when the event collection evolves in time.
This paper only deals with a temporal point process, and therefore, we refer to it as a point process.
We assume that point processes are \emph{simple}, \ie, no events coincide. % and the time stamps can be strictly ordered in time.

\subsubsection{Univariate Point Process}
Assume we observe a sequence of $N\in\bbN$ discrete events during time interval $[0, T]$, and let $\cT$ denote such an observation.
$\cT$ can be represented by a series of event time stamps $\{t_n\in[0,T]\}_{n\in [N]}$ as well as the information that we observe no event during $[0,t_1)$, $\{(t_n, t_{n+1})\}_{n=1}^{N-1}$, and $(t_N, T]$.
Let $\cT^{\leq t_n}$ represents a partial observation of $\cT$ up to and including time $t_n$.
One way of modeling $\cT$ is to specify the probability density function of the event time stamp $t_{n+1}$ given the collection of its past events~$\cT^{\leq t_n}$, which we describe, $f(t\mid \cT^{\leq t_n})$.
Note that the probability density function must satisfy $f(t \mid \cT^{\leq t_n}) = 0$ for $t\leq t_n$ and $\int_{t_n}^\infty f(t \mid \cT^{\leq t_n}) \dl{t} = 1$.
The cumulative distribution function can be defined accordingly: $F(t \mid \cT^{\leq t_n}) = \int_{t_n}^{t} f(s \mid \cT^{\leq t_n}) \dl{s} = \Pr[t_{n+1}\in(t_n, t)\mid \cT^{\leq t_n}]$.

Another way of modeling it is to specify the conditional intensity function, which is related to the distributions as,
\begin{align}
 \label{eq:5} \lambda(t \mid \cT^{\leq t_n}) =
 \begin{cases}
  \dfrac{f\left(t \mid \cT^{\leq t_n}\right)}{1 - F\left(t \mid \cT^{\leq t_n}\right)} & (t > t_n),\\
  0 & (t \leq t_n).
 \end{cases}
\end{align}
In the following, let $t_n$ denote an arbitrary event time stamp and we only specify the conditional intensity function for $t>t_n$, because its value for $t\leq t_n$ is trivially $0$.
%\footnote{When possible, we will often omit the superscript indicating the most recent event, $^{\leq t_n}$, and the conditional intensity function also denotes $\lambda(t \mid \cT)$. We will also omit the definition of $\lambda$ for $t\leq t_n$ when specifying the point process by defining the conditional intensity function.}
%where let $n(t):\bbR_+\rightarrow\bbZ_+$ be a counting process, which represents the number of events up to but not including time $t$ and $t_{n(t)}\in[0,t)$ be the most recent event time stamp in $\cT^{<t}$.
Observing that $\lambda(t \mid \cT^{\leq t_n}) \dl{t} = \Pr[t_{n+1}\in[t, t+\dl{t}] \mid t_{n+1}\notin (t_n, t), \cT^{\leq t_n}]$ holds as $\dl{t}\rightarrow +0$~\cite{rasmussen2018}, 
the conditional intensity function represents how likely the event occurs at time $t$ given that we have observed $n$ events so far and no event has been observed during $(t_n, t)$.

A point process is more often specified by the conditional intensity function than the time interval distribution.
Let $\pp(\lambda)$ be the point process with the conditional intensity function $\lambda$.
Corollary~\ref{prop:legal-cond-intensity}, which is an immediate consequence of Proposition~\ref{prop:mult-point-proc-unique}, states the conditions under which the conditional intensity function uniquely specifies a point process.

\begin{cor}
\label{prop:legal-cond-intensity}
 A conditional intensity function $\lambda$ uniquely defines a point process if it satisfies the following conditions for any observation of discrete events $\cT^{\leq t_n}$ and any $t>t_n$:
\begin{enumerate}
 \item $\lambda(t \mid \cT^{\leq t_n})$ is non-negative and integrable on any interval starting at $t_n$,
 \item $\int_{t_n}^t \lambda(s\mid \cT^{\leq t_n})\dl{s}\rightarrow\infty$ as $t\rightarrow\infty$, and
 \item $\int_{t_n}^t \lambda(s\mid \cT^{\leq t_n})\dl{s}$ is right continuous w.r.t. $t$.
\end{enumerate}
\end{cor}

\begin{figure}[t]
  \centering
 \includegraphics[width=.75\hsize]{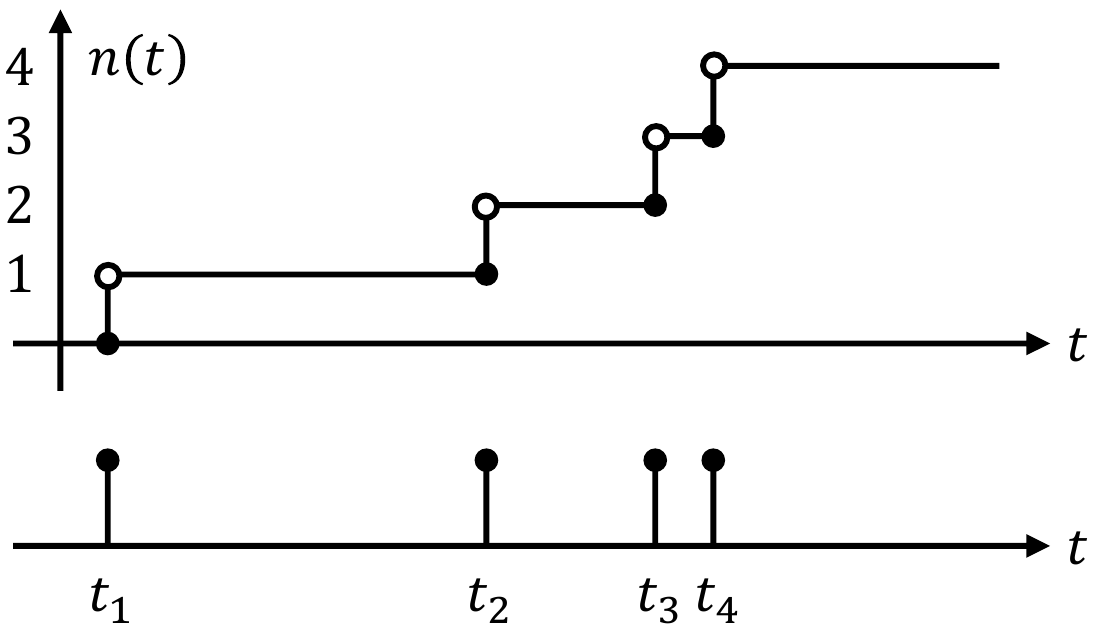}
 \caption{Realization of a temporal point process~(bottom) and its corresponding left-continuous counting process~(top).}\label{fig:counting}
\end{figure}

The log-likelihood of observation $\cT$ on $\pp(\lambda)$ is given as,
\begin{align}
 \label{eq:7}
 \log p(\cT) = \sum_{t\in\cT} \log \lambda\left(t\mid \cT^{\leq t_{n(t)}}\right) - \Lambda^{[0,T]}(\cT),
\end{align}
where let $\Lambda^{[0,T]}(\cT)=\int_0^T \lambda(t\mid \cT^{\leq t_{n(t)}})\dl{t}$ be the integrated conditional intensity function, also known as the \emph{compensator}, which accounts for no-event periods,
and let $n(t):\bbR_{\geq 0}\rightarrow\bbZ_{\geq 0}$ be the left-continuous\footnote{A counting process is usually defined to be right-continuous. We introduce the left-continuous one so as to represent the integrand of the compensator concisely. See Appendix~\ref{sec:relat-stand-notat} for details.} counting process of the observation $\cT$, which counts the number of events up to but not including time $t$.
 The latest event time stamp at time $t$ can be denoted by $t_{n(t)}\in[0,t)$.
 Figure~\ref{fig:counting} illustrates a realization of a point process and its counting representation.
 A typical procedure of modeling $\cT$ is to design a parametric model of the conditional intensity function that satisfies the conditions of Corollary~\ref{prop:legal-cond-intensity} and train it by maximizing the log-likelihood function~(Equation~\eqref{eq:7}).

\subsubsection{Multivariate Point Process}

\begin{figure}[t]
  \centering
 \includegraphics[width=.8\hsize]{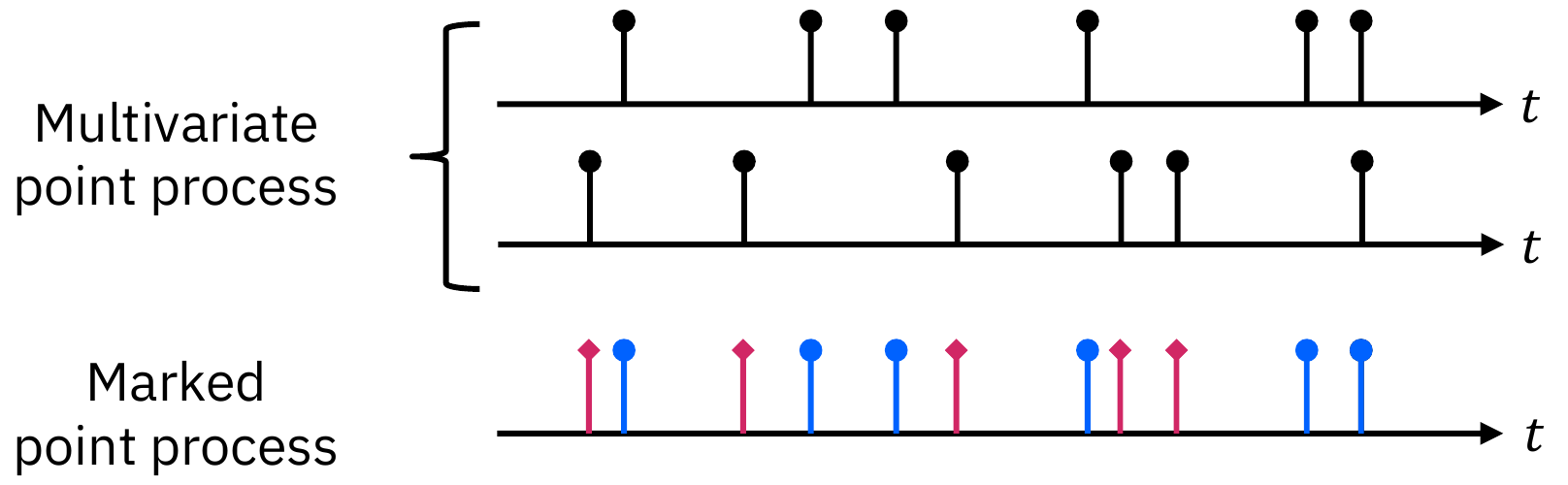}
 \caption{Illustration of a multivariate point process~(top) and its equivalent marked point process~(bottom).}\label{fig:marked-pp}
\end{figure}

A multivariate point process is a set of mutually dependent point processes and can be defined via a \emph{marked point process}, in which each event is associated with a \emph{mark}.
%Let $(X,\cA, \nu)$ be any measure space.
We call a marked point process whose mark belongs to set $X$, an \emph{$X$-marked point process}.
Let $\cT_X$ denote an observation of an $X$-marked point process, which contains a series of event time stamps and marks, $\{(t_n, \bp_n)\in[0, T]\times X\}_{n\in[N]}$.
As illustrated in Figure~\ref{fig:marked-pp}, a $D$-variate point process can be defined by a $\bbone^D$-marked point process,
where each mark $\bp_n$ indicates which dimension the event belongs to. %\footnote{We use the one-hot representation instead of an integer in $[D]$ for ease of presentation.}.
For example, if $\bp_n=\bone_1$, the $n$-th event occurs at the first dimension.
In Figure~\ref{fig:marked-pp}, {blue-circle} and {red-diamond} marks correspond to the first and the second dimensions respectively.
%the first dimension of the multivariate point process is represented by \texttt{blue-circle} mark, and the second one by \texttt{red-diamond} mark.

Letting $f(t, \bp \mid \cT_{\bbone^D}^{\leq t_{n}})$ be the probability density function of each event $(t_{n+1},\bp_{n+1})$ given its past events $\cT_{\bbone^D}^{\leq t_n}$,
the conditional intensity function can be defined similarly:
\begin{align}
 \label{eq:8}\lambda\left(t, \bp \mid \cT_{\bbone^D}^{\leq t_n}\right) =  \dfrac{f\left(t, \bp \mid \cT_{\bbone^D}^{\leq t_n}\right)}{1 - F\left(t \mid \cT_{\bbone^D}^{\leq t_n}\right)},
%\begin{cases}
% \dfrac{f\left(t, \bp \mid \cT_{\bbone^D}^{\leq t_n}\right)}{1 - F\left(t \mid \cT_{\bbone^D}^{\leq t_n}\right)} & (t > t_n),\\
% 0 & (t\leq t_n),
%\end{cases}
\end{align}
where $F(t \mid \cT_{\bbone^D}^{\leq t_{n}}) =  \int_{t_{n}}^t \dl{s} \sum_{\bp \in \bbone^D} f(s, \bp \mid \cT_{\bbone^D}^{\leq t_{n}})$.
The conditional intensity function represents how likely event $(t,\bone_d)$ occurs:
$ \lambda(t,\bone_d \mid \cT_{\bbone^D}^{\leq t_n}) \dl{t}  = \Pr[t_{n+1}\in[t, t+\dl{t}], \bp_{n+1}=\bone_d  \mid t_{n+1}\notin (t_n, t), \cT_{\bbone^D}^{\leq t_n}]$.
Proposition~\ref{prop:mult-point-proc-unique} states conditions under which the conditional intensity function uniquely specifies a marked point process.
See Appendix~\ref{sec:proof-mpp-unique} for its proof.
Let $\mpp(\lambda)$ be the multivariate point process with the conditional intensity function~$\lambda$.
%We discuss a set of conditions that the conditional density function must satisfy so that it defines a marked point process uniquely.
\begin{prop}
 \label{prop:mult-point-proc-unique}
 %Let $(X,\cA,\nu)$ be any measure space.
 Let $X$ be a set. A conditional intensity function $\lambda$ uniquely defines an $X$-marked point process if it satisfies the following conditions
% $\lambda(t,\bp \mid \cT_{A}^{\leq t_n})$ 
 for any $\cT_X^{\leq t_n}$ and $t>t_n$:
\begin{enumerate}
 \item $\lambda(t,\bp \mid\cT_{X}^{\leq t_n})\geq 0$ and integrable w.r.t. $\bp$ and w.r.t. $t$ on any interval starting at $t_n$, 
% \item $\int_{X} \lambda(t,\bp \mid \cT_{X}^{<t}) \dl{\bp}$ is right-continuous with respect to $t$, and
 \item $\int_{t_n}^t \dl{s} \int_{X}\dl{\bp} \lambda(s, \bp\mid \cT_X^{\leq t_n}) \rightarrow\infty$ as $t\rightarrow\infty$, and
 \item $\int_{t_n}^t \dl{s} \int_{X}\dl{\bp} \lambda(s, \bp\mid \cT_X^{\leq t_n})$ is right continuous in $t$.
\end{enumerate}
\end{prop}

%The density function in Eq.~\eqref{eq:6} can be written by the conditional intensity function in Eq.~\eqref{eq:8}, and therefore, these representations can be used interchangeably.
The log-likelihood of observation $\cT_{\bbone^D}$ is written as:
\begin{align}
 \label{eq:12}
%\begin{split}
% &\log p\left(\cT_{\bbone^D}\right) \\
% = &\sum_{\bp\in \bbone^D}\int_0^T \left[\left(\delta \circ \cT_{\bbone^D}\right) (t,\bp) \cdot \log \lambda\left(t,\bp\mid \cT_{\bbone^D}^{\leq t_{n(t)}}\right) \right.\\
% &\hspace{2cm}\left.- \lambda\left(t,\bp \mid \cT_{\bbone^D}^{\leq t_{n(t)}}\right)\right] \dl{t}, 
%\end{split}
\begin{split}
& \log p(\cT_{\bbone^D}) \\
 = &\sum_{(t,\bp)\in\cT_{\bbone^D}} \log \lambda\left(t,\bp\mid \cT_{\bbone^D}^{\leq t_{n(t)}}\right) - \Lambda^{[0,T]}(\cT_{\bbone^D}),
\end{split}
\end{align}
where let $\Lambda^{[0,T]}(\cT_{\bbone^D})=\int_{0}^T \sum_{\bp\in\bbone^D}\lambda(t, \bp \mid \cT_{\bbone^D}^{\leq t_n(t)}) \dl{t}$ be the compensator.
%let $\left(\delta\circ\cT_{\bbone^D}\right)(t,\bp) = \sum_{(t^\prime, \bp^\prime)\in \cT_{\bbone^D}} \delta(t-t^\prime)\cdot\delta_{\bp, \bp^\prime}$ and let $\delta_{\bp, \bp^\prime}$ be the Kronecker delta.
%
%The log-likelihood function requires us to integrate the conditional intensity function to account for non-spiking periods~(the second term in Eq.~\eqref{eq:12}).
Since its analytical form is not available for a general conditional intensity function, we resort to Monte-Carlo approximation to estimate the compensator.
In specific, we draw $M$ examples, $\{t_m\}_{m\in[M]}$, from $U[0,T]$ and approximate it as,
\begin{align}
 %&\sum_{\bp\in\bbone^D}\int_{0}^T \lambda\left(t, \bp \mid \cT_{\bbone^D}^{\leq t_n(t)}\right) \dl{t}\\
 \label{eq:1}\Lambda^{[0,T]}(\cT_{\bbone^D})\approx &\frac{T}{M}  \sum_{m=1}^M \sum_{\bp\in\bbone^D} \lambda\left(t_m, \bp \mid \cT_{\bbone^D}^{\leq t_{n(t_m)}}\right).
\end{align}

\subsection{Sampling Algorithms}\label{sec:sampling-from-point}
This section introduces sampling algorithms for a point process given a conditional intensity function.
A notable feature of the algorithms is that they can exactly simulate point processes without any approximation.
This indicates that there exists no hyperparameter controling the trade-off between computational cost and accuracy of the simulation, which greatly facilitates simulating SNNs.

\subsubsection{Homogeneous Poisson Process}
The simplest point process is the homogeneous Poisson process whose conditional intensity function is constant; $\lambda(t\mid\cT^{\leq t_n})=\lambda$ for any $\cT^{\leq t_n}$.
It is straightforward to sample from it because the interval between two successive events~$\tau$ is independently and identically distributed according to the exponential distribution, $f(\tau; \lambda)=\lambda \mathrm{e}^{-\lambda \tau}$.

\subsubsection{General Point Process}
\begin{algorithm}[t]
 \caption{Thinning algorithm for $\mpp$}\label{alg:thinning-multivariate}
 %\SetKwInOut{Input}{Input}
 %\SetKwInOut{Output}{Output}
 \textbf{Input: }{Conditional intensity function $\lambda$ and upperbound $\bar{\lambda}$}\\
 \textbf{Output: }{Realization of $\mpp(\lambda)$}
 
\begin{algorithmic}[1]
 \STATE $\cS\leftarrow \emptyset$, $\cT\leftarrow \emptyset$\\
 \WHILE{\TRUE}{
 \STATE Sample $s\sim\mathcal{PP}(\bar{\lambda} \mid \cS)$\\
 \IF{$s>T$}{
 \STATE break
 }
 \ENDIF
 \STATE Sample $\begin{bmatrix}
		 \bp \\ r
		\end{bmatrix}\sim\cat\left(\bm{\pi}_{\bar{\lambda}}\circ\bm{\lambda}\left(s \mid \cT\right)\right)$\\
 \IF{$r \neq 1$}{
 \STATE$\cT\leftarrow \cT \cup \{(s, \bp)\}$
 }
\ENDIF
 \STATE $\cS\leftarrow \cS \cup \{s\}$
 }
 \ENDWHILE
\RETURN $\cT$
\end{algorithmic}
\end{algorithm}

It is not straightforward to sample from a general point process when a closed-form expression of the inter-event time distribution is not available.
This is true for many point processes including SNNs.
Among several sampling methods, the thinning algorithm~\cite{doi:10.1002/nav.3800260304,ogata1981} allows us to sample from such a point process without knowing the closed-form expression.
For other sampling algorithms, please refer to Section~\ref{sec:related-work}.

The main idea is to generate a sequence of time stamps from a homogeneous Poisson process with sufficiently high intensity~(which we call the \emph{base process}) and then to thin some of the events so that the sequence follows the given point process.
Algorithm~\ref{alg:thinning-multivariate} describes it for the multivariate case,
where let $\bm{\lambda}\left(t \mid \cT\right) = \begin{bmatrix}\lambda\left(t, \bone_d \mid \cT\right) \end{bmatrix}_{d\in[D]}$, and
let $\bpi_{\bar{\lambda}}\circ$ be an operator that receives a $D$-dimensional vector $\blambda$ and returns $\frac{1}{\bar{\lambda}}\begin{bmatrix}\bm{\lambda} \\ \bar{\lambda} - \left\|\bm{\lambda}\right\|_1 \end{bmatrix}$.

It first generates a new time stamp $s$ from the homogeneous Poisson process with intensity $\bar{\lambda}$~(line~3).
Then it decides whether or not to accept the event, and if accepting, decides which dimension the event is assigned to~(lines~6-8);
$s$ is rejected if $r=1$, \ie,  with probability $1-{\frac{1}{\bar{\lambda}}}{\sum_{\bp\in \bbone^D}\lambda(s, \bp \mid \cT)}$, and
$s$ is accepted as the event from the $d$-th dimension~($d\in[D]$) if $p_d=1$, \ie, with probability $\lambda\left(s, \bone_d \mid \cT\right)/\bar{\lambda}$,

Intuitively, the correctness of Algorithm~\ref{alg:thinning-multivariate} is understood as follows.
Assuming we have sampled $\cT_{\bbone^D}^{\leq t_n}$, at any time $t > t_n$, the probability that the algorithm generates the event with mark $\bone_d$ in interval $[t, t+\dl{t}]$ is,
\begin{align*}
 &\Pr\left[t_{n+1}\in[t, t+\dl{t}], \bp_{n+1}=\bone_d \mid \cT_{\bbone^D}^{<t}\right]\\
 =& \underbrace{\bar{\lambda} \dl{t}}_{\substack{\text{Prob. that the base process}\\ \text{generates the event in } [t, t+\dl{t}]}}
 \cdot \underbrace{\lambda\left(t, \bone_d \mid \cT_{\bbone^D}^{\leq t_n}\right)/\bar{\lambda}}_{\substack{\text{Prob. that }t\text{ is}\\\text{assigned the }d\text{-th mark}}}\\
 =& \lambda\left(t, \bone_d \mid \cT_{\bbone^D}^{\leq t_n}\right) \dl{t},
\end{align*}
%which coincides with the property of the conditional intensity function. %%%%%%%%%%%%%%%%%%%%%%%
where let $\cT_{\bbone^D}^{<t}$ denote the event $t_{n+1}\notin (t_n,t)$ and $\cT_{\bbone^D}^{\leq t_n}$.
This shows that the output follows $\mpp(\lambda)$.
%where $\bar{\lambda}\dl{t}$ is the probability that an event occurs from the base process, and $\lambda\left(t, d \mid \cT_{[D]}^{<t}\right)/\bar{\lambda}$ is the probability that the event is attributed to the $d$-th point process.
For its formal proof, please refer to Reference~\cite{ogata1981}.

\section{Differentiable Point Process}\label{sec:diff-point-proc}
We present the key building block of our method called a \emph{differentiable point process}, whose realization is differentiable with respect to its parameters.
Differentiability plays an essential role when designing a learning algorithm for latent variable models as will be discussed in Section~\ref{sec:learning-algorithms}.

The key idea is that the output of Algorithm~\ref{alg:thinning-multivariate} becomes differentiable if we replace the categorical distribution in line~6 with a reparameterizable distribution such as the concrete distribution, also known as the Gumbel-softmax distribution~\cite{maddison2017,jang2017}.
We first review the concrete distribution~(Section~\ref{sec:concr-distr}), and then we present the differentiable point process~(Section~\ref{sec:mult-diff-point}).

\subsection{Concrete Distribution}\label{sec:concr-distr}
The concrete distribution has been developed as a reparameterizable substitute for the categorical distribution.
The idea comes from the Gumbel-max trick, which enables us to sample from the categorical distribution.
Letting $\bpi\in\bbR_{\geq 0}^D$ be an unnormalized parameter of the categorical distribution, the Gumbel-max trick first samples $u_d\sim U[0,1]$ for each $d\in[D]$, and then outputs $\bone_{d^\star}$ where $d^\star=\arg\max_{d\in[D]} \log \pi_d - \log(-\log u_d)$. The output is known to be distributed according to $\cat(\bpi/\|\bpi\|_1)$.
While the Gumbel-max trick successfully divides the sampling procedure into random sampling from the fixed distribution and a parameterized transformation of it, which is necessary to be differentiable,
the gradient of its realization with respect to $\bpi$ is non-informative, because a small variation to $\bpi$ does not change the gradient.
%In order to turn the gradient informative, we need to think of its continuous relaxation.

The concrete distribution is defined by relaxing the range of the random variable from $\bbone^D$ to its convex hull $\conv{\bbone^D}$ so that its gradient is more informative.
%Noticing that the categorical random variable $d\in[D]$ is equivalently represented by the one-hot vector $\bone_d\in \bbone^D$,
%The continuous relaxation is done by extending the range of the random variable from $\bbone^D$ to its convex hull $\conv{\bbone^{D}}$.
%
Accordingly, the argmax operator in the Gumbel-max trick is replaced with the {softmax} operator with temperature $\tau>0$.
%$\mathrm{softmax}(\bx;\tau)=\begin{bmatrix}\frac{\exp(x_1/\tau)}{\sum_{d=1}^D \exp(x_d/\tau)} & \dots & \frac{\exp(x_D/\tau)}{\sum_{d=1}^D \exp(x_d/\tau)}\end{bmatrix}\in\conv{\bbone^{D}}$.
Since softmax becomes equivalent to argmax as $\tau\rightarrow 0$, the concrete distribution also becomes equivalent to the categorical distribution as $\tau\rightarrow 0$.
Let $g_\tau(\bp;\bpi)$ denote the probability density function of the concrete distribution with temperature $\tau$ and unnormalized parameter $\bpi\in\bbR^{D}_{>0}$.

%In summary, the concrete distribution is an approximation to the categorical distribution, whose realization is differentiable with respect to its parameter.

\subsection{Multivariate Differentiable Point Process}\label{sec:mult-diff-point}

\begin{algorithm}[t]
\caption{Thinning algorithm for $\dpp$} \label{alg:mult-diff-point}
 \textbf{Input: }{Conditional intensity function $\lambda$, its upperbound $\bar{\lambda}$, and temperature $\tau>0$.}\\
 \textbf{Output: }{Realization of $\mathcal{\partial PP}\left(\lambda,\bar{\lambda},\tau\right)$}
 
 \begin{algorithmic}[1]
  \STATE $\cS\leftarrow\emptyset$, $\cT\leftarrow \emptyset$\\
  \WHILE{\TRUE}{
  \STATE Sample $s\sim\mathcal{PP}(\bar{\lambda} \mid \cS)$\\
  \IF{$s>T$}{
  \STATE break
  } \ENDIF
  \STATE Sample $\begin{bmatrix}\bp \\ r \end{bmatrix}\sim\mathrm{Concrete}_\tau\left(\bpi_{\bar{\lambda}}\circ\bm{\lambda}\left(s \mid \cT\right)\right)$\\
 \STATE $\cT\leftarrow \cT \cup \{(s, \bp)\}$\\
  \STATE $\cS\leftarrow \cS \cup \{s\}$
 } \ENDWHILE
 \RETURN $\cT$
 \end{algorithmic}
\end{algorithm}

We present a constructive definition of a differentiable point process in Definition~\ref{def:mult-diff-point}.
%Noticing that Algorithm~\ref{alg:thinning-multivariate} can be seen as a constructive definition of a multivariate point process,
%we will define the differentiable point process based on it.
%

\begin{deff}%[Multivariate Differentiable Point Process]
 \label{def:mult-diff-point}
 Assume the conditional intensity function $\lambda(t, \bp \mid \cT_{\bbone^D}^{\leq t_n})$ can be computed with an observation of a $\convzero{\bbone^D}$-marked point process.
 Let $\bar{\lambda}$ be a constant satisfying $\bar{\lambda}>\sum_{\bp\in\bbone^D} \lambda(t, \bp \mid \cT_{\convzero{\bbone^D}}^{\leq t_n})$ for any $\cT_{\convzero{\bbone^D}}^{\leq t_n}$ and  $t > t_n$,
 and $\tau > 0$ be temperature.
 The differentiable point process $\mathcal{\partial PP}\left(\lambda, \bar{\lambda}, \tau\right)$ is defined as a $\convzero{\bbone^D}$-marked point process constructed by Algorithm~\ref{alg:mult-diff-point}.
\end{deff}
\if0
The main idea is to substitute the concrete distribution for the categorical distribution in line 5 of Algorithm~\ref{alg:thinning-multivariate}, resulting in Algorithm~\ref{alg:mult-diff-point}.
%
In Algorithm~\ref{alg:thinning-multivariate}, $r_k$ indicates whether $s_k$ should be discarded or not,
and if not, mark $\bp_k\in\bbone^D$ is assigned to the event time stamp $s_k$.
%
In Algorithm~\ref{alg:mult-diff-point}, each event is always accepted, and
$\bp_k\in\convzero{\bbone^{D}}$ is used as a mark of the event $s_k$.
\fi

Algorithms~\ref{alg:thinning-multivariate} and \ref{alg:mult-diff-point} are different in two ways. %the original multivariate point process and its differentiable counterpart.
First, all events from the base process are accepted in Algorithm~\ref{alg:mult-diff-point}, while some are rejected in Algorithm~\ref{alg:thinning-multivariate}.
Second, in Algorithm~\ref{alg:thinning-multivariate}, the mark is defined over $\bbone^D$, while in Algorithm~\ref{alg:mult-diff-point}, it is defined over $\convzero{\bbone^D}$; each mark is associated with {amplitude} that is continuous w.r.t. the model parameter.

The differentiable point process as defined above can be understood as a marked point process~(Proposition~\ref{prop:diff2mark}).
\begin{prop}
 \label{prop:diff2mark}
 The differentiable point process $\mathcal{\partial PP}(\lambda,\bar{\lambda},\tau)$ is a $\convzero{\bbone^D}$-marked point process with conditional intensity function,
\begin{align*}
 &\lambda_{\partial}\left(t, \bp \mid \cT_{\convzero{\bbone^{D}}}^{\leq t_n};\blambda, \bar{\lambda}, \tau\right) \\
 = &
 \bar{\lambda} \cdot g_\tau\left(
\begin{bmatrix} 
\bp\\ 1- \|\bp\|_1 
\end{bmatrix}; 
\bpi_{\bar{\lambda}}\circ\bm{\lambda}\left(t \mid \cT_{\convzero{\bbone^D}}^{\leq t_n}\right)\right).
\end{align*}
\if0
 where let,
\begin{align*}
 \bm{\lambda}\left(t \mid \cT_{\convzero{\bbone^D}}^{\leq t_n}\right) &= \begin{bmatrix}\lambda\left(t, \bone_d \mid \cT_{\convzero{\bbone^D}}^{\leq t_n}\right) \end{bmatrix}_{d\in[D]},\\
 \bpi_{\bar{\lambda}}\circ\bm{\lambda} &= \frac{1}{\bar{\lambda}}\begin{bmatrix}\bm{\lambda} \\ \bar{\lambda} - \left\|\bm{\lambda}\right\|_1 \end{bmatrix}.
\end{align*}
\fi
\end{prop}

We can confirm the differentiability of a realization of $\partial\mathcal{PP}$~(Proposition~\ref{prop:differentiability}).
We can also confirm that in the limit of $\tau\rightarrow 0$, the differentiable point process becomes equivalent to the original point process~(Proposition~\ref{prop:convergence}).
See Appendix~\ref{appendix:prop-diff-point} for their formal statements and proofs.

As discussed by \citeauthor{maddison2017}~(\citeyear{maddison2017}), the concrete distribution often suffers from underflow and we have to implement it in the logarithmic scale.
Our implementation also suffers from the same issue, and we provide a numerically stable implementation idea in Appendix~\ref{sec:numer-stable-impl}.

\section{Learning Algorithm for SNNs}

We present a learning algorithm for spiking neural networks~(SNNs) based on the differentiable point process.
We first define a probabilistic model of SNNs~(Section~\ref{sec:probabilistic-model}) and then will present our learning algorithm, highlighting the difference from the existing one~(Section~\ref{sec:learning-algorithms}).

\subsection{Probabilistic Model of Spiking Neural Networks}\label{sec:probabilistic-model}
We employ the standard probabilistic model in the literature~\cite{doi:10.1162/neco.2006.18.6.1318}.
%A sequence of spikes emitted by a single neuron is referred to a \emph{spike train}.
%We model spike trains emitted from SNN as a realization of a multivariate point process. %; thus, we only care about time stamps of spikes, and ignore the spike shapes.
Let $D$ be the number of neurons, let $\cN=\bbone^D$ be the set of neurons, each of which is indexed by a one-hot vector,
and let $\cT_{\cN}$ be spike trains emitted from SNN during time interval $[0, T]$. %\ie, a sequence of time stamps of spikes emitted by the neurons.
We assume that $\cT_{\cN}$ is a realization of an $\cN$-marked point process. %, whose conditional intensity function will be defined in the following.
%Assume that there exists a conditional intensity function $\lambda_\theta\left(t, \bp \mid \cT_{\cN}\right)$ that governs the spike trains.

We define the conditional intensity function based on a \emph{spike response model}~(SRM)~\cite{gerstner2014}. %, which describes the internal dynamics of a spiking neuron.
SRM assumes that the $d$-th spiking neuron is driven by its internal state called a membrane potential,
\begin{align}
 \label{eq:11}
\begin{split}
 u_d\left(t \mid \cT_{\cN}^{\leq t_n}\right) = \bar{u}_d + \sum_{(t^\prime, \bp)\in\cT_{\cN}^{\leq t_n}} \mathbf{f}_{d}(t-t^\prime)\cdot \bp,
% \int_{\bp\in\cN} \dl{\bp} \int_{0}^{t} \dl{t^\prime} \left(\delta\circ\cT_{\cN}^{\leq t_n}\right)(t^\prime, \bp)  \left[\mathbf{f}_{d}(t-t^\prime)\cdot \bp\right], 
\end{split}
\end{align}
where $\mathbf{f}_d(s) = \begin{bmatrix}f_{d^\prime,d}(s) \end{bmatrix}_{d^\prime\in[D]}$ is a vector of filter functions from all of the neurons to the $d$-th neuron.
In specific, $f_{d^\prime, d}(s)$ describes the time course of the membrane potential of neuron $d$ in response to a spike emitted by neuron $d^\prime$ at time $s=0$.
We assume $f_{d, d}(s)\leq 0$ for all $d\in\cN$.
This assumption allows us to reproduce the resetting behavior of a biological neuron; the membrane potential is reset to a lower level after the neuron fires.
We also assume that $f_{d^\prime, d}(s)=0$ for $s<0$. This assumption ensures that future events have no influence on past events.

Then, the conditional intensity function is defined by,
\begin{align}
 \label{eq:10}\lambda^{\mathrm{SNN}}(t, \bp \mid \cT_{\cN}^{\leq t_n}) =  \bp \cdot \sigma(\bu(t \mid \cT_{\cN}^{\leq t_n})),
%\begin{cases}
% \bp \cdot \sigma\left(\bu\left(t \mid \cT_{\cN}^{\leq t_n}\right)\right) & (t> t_n),\\
% 0 & (t \leq t_n),
%\end{cases}
\end{align}
where $\sigma\colon\bbR^D\rightarrow\bbR_{\geq 0}^D$ is element-wisely non-decreasing and differentiable\footnote{We use the sigmoid function multiplied by amplitude $a>0$ element-wisely as $\sigma$, for which $\bar{\lambda}$ is easy to derive.} and let $\bu(t \mid \cT_{\cN}^{\leq t_n}) = \begin{bmatrix} u_d(t \mid \cT_{\cN}^{\leq t_n}) \end{bmatrix}_{d\in[D]}$. % is the vector of all the membrane potentials.
As the membrane potential of one neuron increases, the neuron is more likely to fire and generate a spike.

For numerical simulation, we assume that the filter functions are parameterized by weights $\{w_{d^\prime,d,l}\in\bbR\}_{l=1}^L$ as,
\begin{align}
 \label{eq:21}f_{d^\prime,d}(s) =
 \begin{cases}
   \sum_{l=1}^L w_{d^\prime, d, l} \cdot \kappa(s-s_l) & (s\geq 0),\\
  0 & (s < 0),
 \end{cases}
\end{align}
where $\{s_l\in\bbR\}_{l=1}^L$ are fixed and $\kappa(s)=\max\{\frac{3}{4}(1-s^2), 0\}$ is the Epanechnikov kernel.
We chose this kernel because the bounded support of the kernel allows us to ignore events that occurred more than a certain period ago for membrane potential computation.
Let $\theta=\{\bar{u}_d\in\bbR\}_{d=1}^D\cup\{w_{d^\prime, d, l}\in\bbR \mid l\in[L]\}_{d, d^\prime=1}^D$ denote the set of model parameters.

\subsection{Learning Algorithms}\label{sec:learning-algorithms}

Assume some of the neurons are hidden and their spike trains are unobservable.
Let $\cO\subset\cN$ and $\cH=\cN\backslash\cO$ be the sets of observable and hidden neurons, respectively.
Accordingly, the spike trains of all of the neurons are divided into observable and hidden ones: $\cT_{\cN}=\cT_{\cO}\cup\cT_{\cH}$.
We consider an estimation procedure for the model parameters of SNN, $\theta$, given a set of observed spike trains~$\left\{\cT_{\cO, n}\right\}_{n=1}^N$.

Letting $p(\cT_\cN;\theta)=p(\cT_{\cO}, \cT_{\cH};\theta)$ be the joint distribution of the observable and hidden spike trains,
the parameter $\theta$ is estimated by maximum likelihood estimation:
\begin{maxi*}|l|[3]
{\theta}{\sum_{n=1}^N \ell\left(\theta; \cT_{\cO,n}\right)}
{}{}
\end{maxi*}
where $\ell(\theta;\cT_{\cO}) = \log \int p(\cT_{\cO}, \cT_{\cH}; \theta) \dl \cT_\cH$ is the marginalized log-likelihood function.
Since it is intractable to compute it, we substitute its lower bound called an \emph{evidence lower bound}~(ELBO) for the marginalized log-likelihood function as the objective function:
\begin{align}
\nonumber  \underline{\ell}(\theta,\phi;\cT_{\cO}) &= \bbE_{q(\cT_{\cH};\phi)}\left[\log p(\cT_{\cO}, \cT_{\cH}; \theta) - \log q(\cT_{\cH};\phi)\right],\\
\label{eq:14} &\equiv \bbE_{q(\cT_{\cH};\phi)} \underline{\ell}(\theta, \phi;\cT_{\cO}, \cT_{\cH}),
\end{align}
where $q(\cT_{\cH};\phi)$ is an arbitrary distribution called a \emph{variational distribution}, parameterized by $\phi$.
We specifically assume that the variational distribution is modeled by SNN driven by both observable and hidden spike trains.
In the following, we omit the index of data $n$ for ease of presentation and consider ELBO using a single observation $\cT_{\cO}$.

Since there exists no closed-form solution to the maximization problem, we resort to stochastic gradient ascent methods, resulting in Algorithm~\ref{alg:part-observ-case}.
The basic procedure to train SNN is to choose one realization $\cT_{\cO}$ from the data set randomly, and update $\theta$ and $\phi$ so as to maximize Equation~\eqref{eq:14}.
In the following, we present both an existing approach and our novel approach to compute the gradients, $\frac{\partial\underline{\ell}}{\partial\theta}$ and $\frac{\partial\underline{\ell}}{\partial\phi}$.

\begin{algorithm}[t]
 \caption{Generic learning algorithm}\label{alg:part-observ-case}
 \textbf{Input: }{Observation $\cT_\cO$, learning rate $\{\alpha_k\}_{k=1}^{K}$.}\\
 \textbf{Output: }{Model parameters  $\theta$, $\phi$.}

 \begin{algorithmic}[1]
  \STATE Initialize $\theta$, $\phi$\\
  \FOR{$k=1,\dots,K$}{
  \STATE Update $\theta\leftarrow\theta + \alpha_k \frac{\partial\underline{\ell}}{\partial\phi}(\theta,\phi;\cT_\cO)$\\
  \STATE Update $\phi\leftarrow\phi + \alpha_k \frac{\partial\underline{\ell}}{\partial\phi}(\theta,\phi;\cT_\cO)$ %\\
  %\STATE Increment $k\leftarrow k+1$
  }\ENDFOR
  \RETURN {$\theta$, $\phi$}  
 \end{algorithmic}
\end{algorithm}

\subsubsection{Gradient with respect to $\theta$}
The gradient with respect to $\theta$ is straightforwardly computed by applying Monte-Carlo approximation:
%\begin{align*}
$\diffp*{\underline{\ell}(\theta,\phi;\cT_\cO)}{\theta} \approx \hat{\bbE}_{q(\cT_{\cH};\phi)}\left[\diffp*{\log p(\cT_{\cO}, \cT_{\cH}; \theta)}{\theta}\right]$.
% \label{eq:16}&\approx\dfrac{1}{M}\sum_{m=1}^M \dfrac{\partial}{\partial\theta}\log p(\cT_{\cO}, \cT_{\cH,m}; \theta),
%\end{align*}
%where let $\left\{\cT_{\cH,m}\right\}_{m=1}^M$ be i.i.d. sample of size $M$ from $q(\cT_{\cH};\phi)$.
This can be numerically calculated with the help of automatic differentiation tools.
%The expectation in Eq.~\eqref{eq:15} is difficult to compute, and we use Monte-Carlo estimation.
\if0
Letting $\cT_{\cH,m}^{<T}$~($m\in[M]$) be i.i.d. sample of size $M$ from $q(\cT_{\cH}^{<T};\phi)$, the gradient is approximated as,
\fi

\subsubsection{Gradient with respect to $\phi$}
\label{sec:grad-with-resp-phi}
The gradient with respect to $\phi$ is more involved. In Equation~\eqref{eq:14}, the expectation operator depends on $\phi$ and we cannot exchange $\frac{\partial}{\partial\phi}$ and $\bbE_{q(\cT_\cH;\phi)}$.
There are at least two approaches to computing the gradient in this situation~\cite{mohamed2019monte}.
One approach is to rely on \emph{the score function gradient estimator}, also known as the REINFORCE estimator~\cite{Williams1992}.
While it is widely applicable to a variety of models, it is often reported that the gradient estimator has high variance.
Another approach is \emph{the path-wise gradient estimator}, which makes use of the reparameterization trick~\cite{kingma2014}.
While its variance is often reported to be lower than that of the score function gradient estimator~\cite{mohamed2019monte},
its application is limited because the probability distribution $q$ must be reparameterizable.

In the literature of SNNs, the score function gradient estimator with respect to $\phi$ has been developed by \citeauthor{10.3389/fncom.2014.00038}~(\citeyear{10.3389/fncom.2014.00038}).
Our contribution is to develop a path-wise gradient estimator for SNNs based on a differentiable point process presented in Section~\ref{sec:diff-point-proc}.

\textbf{Score function gradient estimator.}
\citeauthor{10.3389/fncom.2014.00038}~(\citeyear{10.3389/fncom.2014.00038}) used the score function gradient estimator for computing the gradient with respect to $\phi$:
%\begin{align*}
$\frac{\partial\underline{\ell}}{\partial\phi}(\theta, \phi;\cT_\cO) \approx\hat{\bbE}_{q(\cT_\cH;\phi)}[\frac{\partial\log q(\cT_\cH;\phi)}{\partial\phi} (\underline{\ell}(\theta, \phi;\cT_{\cO}, \cT_{\cH}) - 1)]$.
% =& \frac{\partial}{\partial\phi}\bbE_{q(\cT_\cH;\phi)}\left[\log p(\cT_\cO,\cT_\cH;\theta) - \log q(\cT_{\cH};\phi)\right]\\
%=&\frac{\partial}{\partial\phi} \int q(\cT_\cH;\phi) \left[\log p(\cT_\cO,\cT_\cH;\theta) - \log q(\cT_{\cH};\phi)\right] \dl \cT_{\cH}\\
% \approx&\frac{1}{M}\sum_{m=1}^M \left[\frac{\partial}{\partial\phi}\left(\log q(\cT_{\cH,m};\phi)\right) \left(\underline{\ell}(\theta, \phi;\cT_{\cO}, \cT_{\cH,m}) - 1\right)\right],
%\end{align*}
%where let $\left\{\cT_{\cH,m}\right\}_{m=1}^M$ be i.i.d. sample of size $M$ from $q(\cT_{\cH};\phi)$.
While this is an unbiased estimator of the gradient, its high variance is often problematic.
We employ the variational distribution with the conditional intensity function,
\begin{align}
 \label{eq:19}\lambda_q(t, \bp \mid \cT_{\cN}^{\leq t_n};\phi) = \bp \cdot \sigma(\bu(t \mid \cT_{\cN}^{\leq t_n};\phi)),
%\begin{cases}
% \bp \cdot \sigma(\bu(t \mid \cT_{\cO\cup\cH}^{\leq t_n};\phi)) & (t > t_n),\\
% 0 & (t\leq t_n),
%\end{cases}
\end{align}
for any $\bp\in\cH$. % and any $\cT_{\cO\cup\cH}^{\leq t_n}$, which contains spike trains from neurons $\cO$ and $\cH$.
In particular, we use shared parameters for the model and the variational distribution, \ie, we set $\phi=\theta$ as we observe it improves the performance.

\textbf{Path-wise gradient estimator.}
We propose a path-wise gradient estimator for SNNs.
%Let us first provide the overview of our approach.
%The path-wise gradient estimator requires the variational distribution to be reparameterizable. %; \ie, its realization is differentiable with respect to the parameter of the variational distribution, $\phi$.
Our main idea is to employ the differentiable point process, $\mathcal{\partial PP}(\lambda_q(t, \bp \mid \cT_\cN;\phi); \bar{\lambda}, \tau)$, as the variational distribution, where $\lambda_q$ is defined in Equation~\eqref{eq:19}. %, $\bar{\lambda}$ is an upperbound of the conditional intensity function, and $\tau$ is temperature.
This allows us to differentiate a Monte-Carlo approximation of ELBO~(Equation~\eqref{eq:14}) using automatic differentiation tools:
%with an i.i.d. sample of size $M$, $\{\cT_{\cH,m}(\phi)\}_{m=1}^M$, from the variational distribution, ELBO is approximated by,
\begin{align}
 \label{eq:20}\diffp{\underline{\ell}(\theta,\phi;\cT_{\cO})}{\phi} \approx \diffp{\hat{\bbE}_{\dpp} \underline{\ell}(\theta,\phi;\cT_{\cO}, \cT_{\convzero{\cH}}(\phi))}{\phi}.
 %\frac{1}{M} \sum_{m=1}^M \underline{\ell}(\theta, \phi;\cT_\cO, \cT_{\cH,m}(\phi)),
% \left[\log p(\cT_{\cO}, \cT_{\cH,m}(\phi); \theta) - \log q(\cT_{\cH,m}(\phi);\phi)\right],
\end{align}
%which will be differentiable with respect to $\phi$.

The main technical issue in applying the differentiable point process is that its realization $\cT_{\convzero{\cH}}(\phi)$ is incompatible with the SNN model defined by Equations~\eqref{eq:11} and \eqref{eq:10}.
The model assumes that a mark $\bp$ is a one-hot vector, %\ie, $\bp\in\cN=\bbone^D$,
while a mark of a differentiable point process belongs to $\convzero{\cH}$.
We address this by devising a \emph{differentiable spiking neural network}~($\partial$SNN),
%by extending the SNN model so that it
which can handle a mark in $\convzero{\cH}$, while keeping the conditional intensity function proper.

Let $\bar{\cN}=\cO\cup\convzero{\cH}$ be the set of marks for $\partial$SNN.
We define the membrane potential of neuron $d\in\cN$ as,
\begin{align}
 \label{eq:18} & u_d\left(t \mid \cT_{\bar{\cN}}^{\leq t_n}\right) = \bar{u}_d + \sum_{(t^\prime, \bp)\in\cT_{\bar{\cN}}^{\leq t_n}} \mathbf{f}_{d}(t-t^\prime)\cdot \bp,
\end{align}
and the conditional intensity of $\partial$SNN for $\bp\in\bar{\cN}$ as,
  \begin{align}
 %\int_{\bp\in\bar{\cN}} \dl{\bp} \int_{0}^{t} \dl{t^\prime} \left(\delta\circ\cT_{\bar{\cN}}^{\leq t_n}\right)(t^\prime, \bp)  \left[\mathbf{f}_{d}(t-t^\prime)\cdot \bp\right],\\
 \label{eq:17}&\lambda^{\partial\mathrm{SNN}}\left(t, \bp \mid \cT_{\bar{\cN}}^{\leq t_n}; \bar{\lambda}, \tau\right)\\
\nonumber =& \sum_{\bone_d\in\cO}\delta(\bp-\bone_d) \lambda^{\mathrm{SNN}}\left(t, \bp \mid \cT_{\bar{\cN}}^{\leq t_n}\right)\\
\nonumber &+ \mathbb{I}[\bp\in\convzero{\cH}] \lambda_{\partial}\left(t, \bp_\cH \mid \cT_{\bar{\cN}}^{\leq t_n};\blambda_{\cH}, \bar{\lambda}, \tau\right)   
  \end{align}
where $\blambda_{\cH}\left(t \mid \cT_{\bar{\cN}}^{\leq t_n}\right) = \sigma\left(\begin{bmatrix}u_{d}(t \mid \cT_{\bar{\cN}}^{\leq t_n})\end{bmatrix}_{d\in\cH}\right)$, $\mathbb{I}[\cdot]$ is the indicator function, and $\bp_\cH = \begin{bmatrix}p_d \end{bmatrix}_{d\in\cH}$.
%Note that the conditional intensity function has both discrete and continuous elements.

It is necessary to confirm that (i) the conditional intensity function can be calculated using past events whose marks are in $\bar{\cN}$
and (ii) the conditional intensity function satisfies all of the conditions listed in Proposition~\ref{prop:mult-point-proc-unique} for $\cX=\bar{\cN}$.
The first requirement immediately follows from Equations~\eqref{eq:18}~and~\eqref{eq:17}.
%allows us to compute the conditional intensity function using a realization of the differentiable point process, satisfying the first requirement.
In Appendix~\ref{appendix:prop-diff-spik}, we provide the formal statement and proof of the second requirement~(Proposition~\ref{prop:uniqueness}).
We also confirm that ELBO is differentiable~(Proposition~\ref{prop:differentiable-log-likelihood}) and that the differentiable SNN becomes equivalent to the vanilla SNN in the limit of $\tau\rightarrow 0$~(Proposition~\ref{prop:converge-snn}).

\section{Empirical Studies}\label{sec:empirical-study}

\begin{table}[t]
 \caption{Configuration of SNN generating a synthetic data set.}\label{tab:snn-config}
 \centering
\small
\begin{tabular}[t]{ll}
\toprule
 Network size & $D=6$, $|\cO|=2$, $|\cH|=4$ \\
 Activation/filter functions & $a=5$, $L=2$, $s_1=0$, $s_2=10$\\
 $\dpp$ & $\tau=0.3$, $\bar{\lambda}=20$\\
 \# of samplings & $100$~(Eq.~\eqref{eq:1}), $1$~(Eq.~\eqref{eq:14})\\
\bottomrule
\end{tabular}
\end{table}

Let us investigate the effectiveness of our gradient estimator through numerical simulation.
Our hypothesis is that (i)~the path-wise gradient estimator will have lower variance than the score function estimator and (ii) lower variance will improve the predictive performance.
We design two experiments~(Sections~\ref{sec:vari-grad-estim} and \ref{sec:pred-perf}) to verify these two hypotheses.
We additionally compare computation cost of the learning algorithms using each of the gradient estimators in Section~\ref{sec:comp-time}.
All the experiments are conducted on IBM Cloud\footnote{Intel Xeon Gold 6248 2.50GHz 48 cores and 192GB memory.}, and the code is publicly available~\cite{DIffSNNGithub}.

%\subsection{Settings}

\textbf{Data set.} We use a synthetic data set generated by the vanilla SNN~(Equation~\eqref{eq:10}).
Table~\ref{tab:snn-config} summarizes its configuration.
%We use a minimal network because we mainly focus on the learning capability rather than the scalability to the network size.
We set $\bar{\lambda}=a|\cH|=20$,
which is the tightest upperbound because we use the sigmoid activation function with amplitude $a$.
The weights are randomly sampled: biases from $U[-1,1]$, off-diagonal kernel weights from $U[-5, 5]$, and diagonal kernel weights from $U[-5, -0.1]$.

\textbf{Methods compared.}
Since our objective is to highlight the performance gap between our path-wise gradient estimator~($\partial$SNN) and the score function gradient estimator~(SNN),
we use the same hyperparameters and initialization for both of them as much as possible.
We initialize their parameters randomly using the same random seed so that both of them have random but the same initial parameters.
We also set their hyperparameters as Table~\ref{tab:snn-config}.
The temperature is the only hyperparameter that impacts the performance gap.
In preliminary experiments, we observe no significant impact for $\tau\in[0.1, 0.5]$, and we only report the result at $\tau=0.3$.

\if0

For the ground-truth network topology, letting the first two dimensions be observable and the last two dimensions be hidden,
we consider a fully-connected architecture with the following weight and bias parameters:
\begin{align}
 \bar{\bu} &= \begin{bmatrix}
	       0 & 0 & 0 & 0
	      \end{bmatrix}^\top,\\
 W_1 &= \begin{bmatrix}
	 0 & 0 & 0 & 0\\
	 0 & 0 & 0 & 0\\
	 0 & 0 & 0 & 0\\
	 0 & 0 & 0 & 0
	\end{bmatrix},\\
 W_2 &= \begin{bmatrix}
	 0 & 0 & 0 & 0\\
	 0 & 0 & 0 & 0\\
	 0 & 0 & 0 & 0\\
	 0 & 0 & 0 & 0
	\end{bmatrix},
\end{align}
where the $(d^\prime, d)$-th element of $W_l$ corresponds to $w_{d^\prime, d, l}$ for $d, d^\prime\in[D]$ and $l=1, 2$.

For SNN to be trained, we initialize the bias term uniform randomly from $[0, 1]$ and the diagonal component of $W_l$ from $[-1, -0.1]$.
We set the upperbound of the diagonal components to be $-0.1$ and at each parameter update we clip them so that they satisfy the constraint.
\fi

\subsection{Variance of the Gradient Estimators}\label{sec:vari-grad-estim}
First, let us study the variance of the gradient estimators.

\textbf{Protocol.}
We generate a single random parameter setting and use it to generate a synthetic data set consisting of $10$ examples of length $50$.
Then, we compute the gradient estimators using the whole data set $1000$ times, which yields $1000$ gradient estimates for each method.
Finally, we compute the standard deviations of each element of the gradients, and report the mean of the standard deviations.

\textbf{Result.}
The mean standard deviation of $\partial$SNN was $66.3$, whereas that of SNN was $2.49\times 10^3$.
This clearly demonstrates that the variance of our estimator tends to be lower than that of the existing estimator.

\subsection{Predictive Performance}\label{sec:pred-perf}
The second experiment studies the predictive performance of the models learned by each of the methods compared.

\textbf{Protocol.}
We generate $24$ random parameter settings, and consistently use them in this experiment.
We aim to evaluate the performance gap between SNN and $\partial$SNN in different sizes of training sets.
To this end, we execute the following, varying the size as $N_{\mathrm{train}}=10$, $20$, $30$, $40$, $50$, $75$ $100$, $200$, and for each parameter setting.

%For each parameter set, we execute the following.
We generate training/test sets consisting of $N_{\mathrm{train}}$/$100$ examples of length $50$ respectively.
SNN and $\partial$SNN are trained on the training set using AdaGrad~\cite{Duchi2011} with initial learning rate $0.05$ for $10$ epochs.
We evaluate the predictive performance by computing ELBO~(Equation~\eqref{eq:14}) on the test set.
For fair comparison, we evaluate the performance of $\partial$SNN by transferring its parameters to the vanilla SNN\footnote{For better transfer, we decrease $\tau$ geometrically by ratio $0.95$ at every epoch, which slightly improves the performance.}.
By repeating this over $24$ parameter settings, we obtain $24$ ELBO scores.
We report their mean as the performance of each method for each $N_{\mathrm{train}}$.

\begin{figure}[t]
\centering
 \includegraphics[width=.78\hsize]{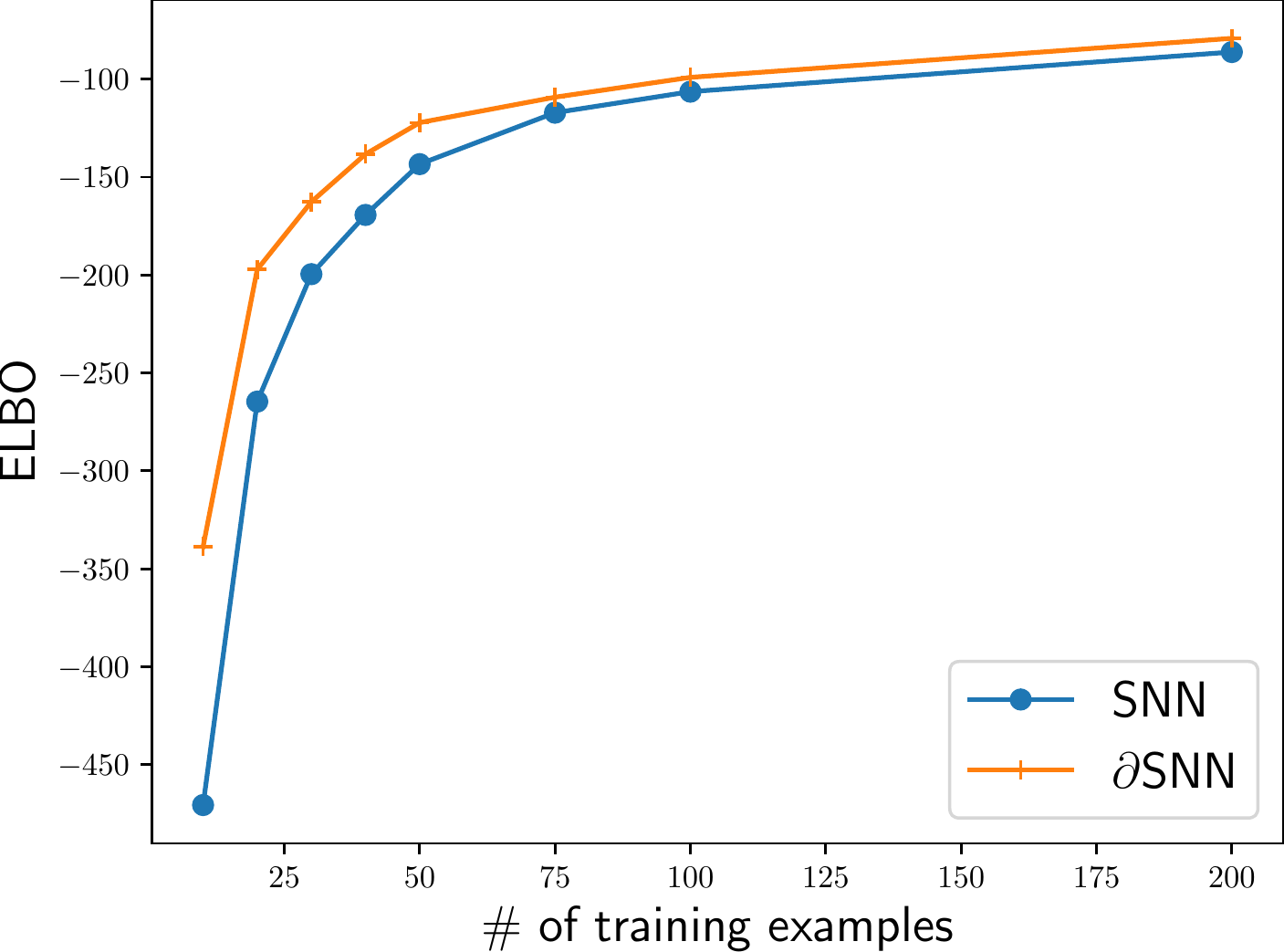}
\caption{Predictive performance of SNN and $\partial$SNN.}
\label{fig:predictive-performance}
\end{figure}

\textbf{Result.} Figure~\ref{fig:predictive-performance} summarizes the experimental results.
It clearly shows that $\partial$SNN consistently outperforms SNN especially in the small-sample regime, which supports the benefit of our low-variance estimator.
%While we observe diminishing returns as we increase the number of training examples, we observe consistent performance improvement, which supports the benefit of our low-variance estimator.

\subsection{Computational Overhead}\label{sec:comp-time}
The last experiment studies computation overhead of $\partial$SNN over SNN.
The computation time depends on the number of spikes, and the number of (hidden) spikes is proportional to $a$, the amplitude of the non-linearlity~$\sigma$ that maps the membrane potential into the conditional intensity function.
In general, $\partial$SNN generates more hidden spikes than SNN because the thinning algorithm for the differentiable point process does not reject any of the candidate spikes.
Therefore, we expect that $\partial$SNN requires more computation time than SNN. The purpose of this experiment is to measure the computational overhead of $\partial$SNN over SNN.

\textbf{Protocol.}
We generate a single parameter setting, and generate a training set of $10$ examples of length $50$.
We then set up both SNN and $\partial$SNN with amplitude $a=1,2,\dots,20$, resuting in $40$ models to be trained.
For each model, we measure the computation time of running $100$ epochs, and obtain per-epoch computation time by averaging them.

\begin{figure}[t]
 \centering
 \includegraphics[width=.78\hsize]{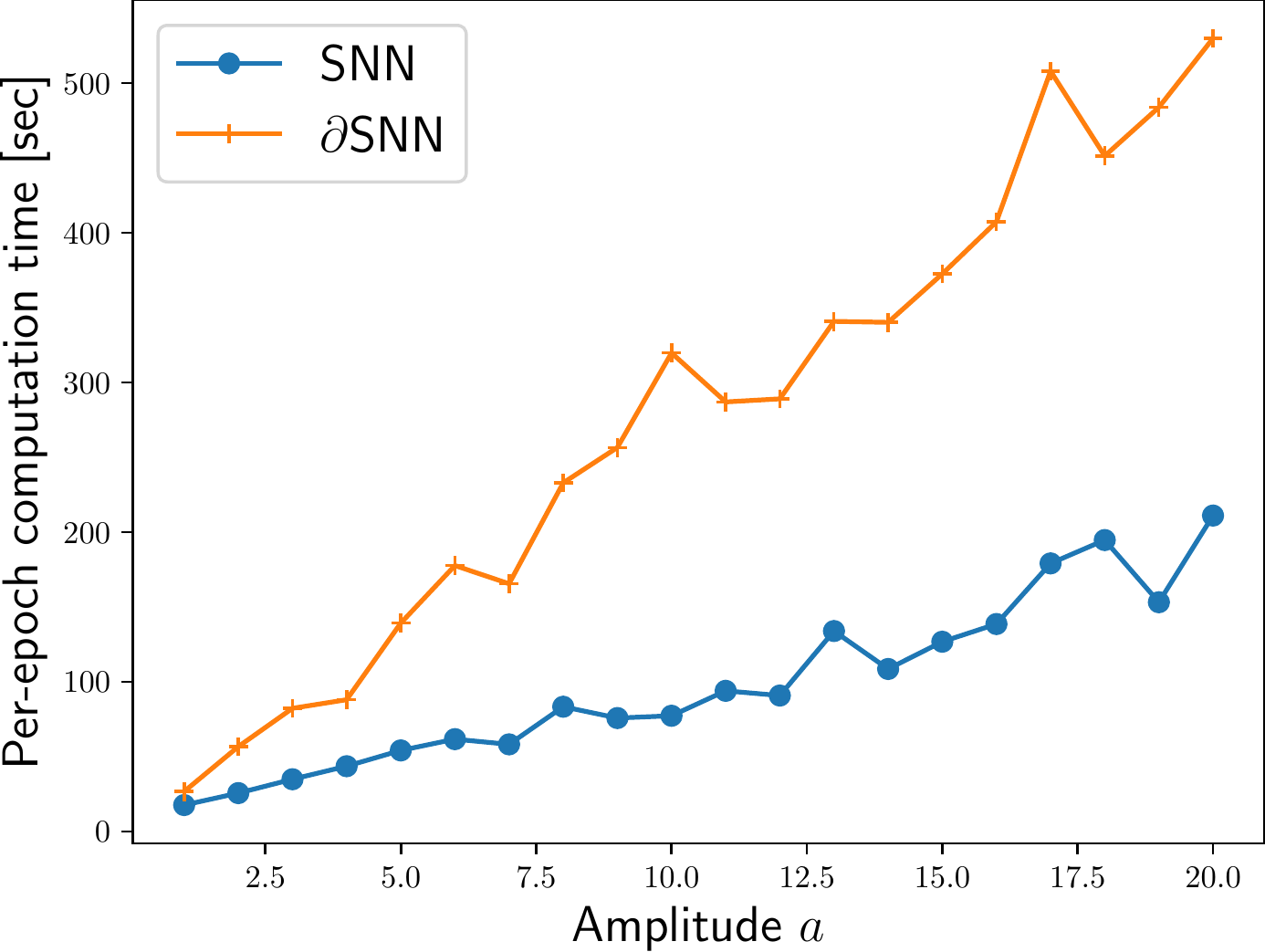}
 \caption{Per-epoch computation time of SNN and $\partial$SNN.}\label{fig:computation-time}
\end{figure}

\textbf{Result.}
Figure~\ref{fig:computation-time} summarizes the experimental results.
As is expected, $\partial$SNN requires $2.8$ times more computation time than SNN on average.
This result can be used as a reference for users to decide which gradient estimator to be employed.
If a user can afford this overhead, our path-wise gradient estimator is recommended; otherwise, please consider to use the score function gradient estimator.

Note that we can improve the computation time of our method by introducing an adaptive upperbound $\bar{\lambda}$ in Algorithm~\ref{alg:part-observ-case}, if it is a tigher upperbound than the fixed upperbound. We leave this improvement as future work.

\section{Related Work}\label{sec:related-work}
The present work is related to the communities of SNNs and point processes.
Let us discuss our contributions to them.

\subsection{Spiking Neural Networks}
The most relevant work is the stochastic variational learning algorithm for SNNs~\cite{10.3389/fncom.2014.00038}.
As discussed in Section~\ref{sec:grad-with-resp-phi}, the difference is the gradient estimator. % for the variational distribution.
The authors used the score function gradient estimator, because the path-wise gradient estimator~(which became popular by VAE~\cite{kingma2014}) was not popular at that time and the reparameterization trick for point processes was not trivial. % and the reparameterization trick for a point process had not been discussed until recently~\cite{Shchur2020}.
%It is not trivial to develop a path-wise gradient estimator for SNNs, because a reparameterizable point process specified by the conditional intensity function is not known in the community.
Our contribution is to develop a differentiable point process that enables us to derive the path-wise gradient estimator.
\if0
Since it is widely acknowledged that the reparameterization trick has often empirically less variance than the score function estimator,
we develop a reparameterizable point process and apply it to the variational learning algorithm for SNNs.
In our empirical study, we empirically validate that our algorithm achieves better generalization ability than that reported in Reference~\cite{10.3389/fncom.2014.00038}.
\fi

Less relevant but still worth mentioning are the line of work in learning algorithms for deterministic SNNs, where a neuron fires when the membrane potential exceeds a threshold.
%Since they do not have hidden random variables,
Although our technique cannot directly contribute to them,
we believe it is worthwhile to compare the pros and cons of these different approaches for further development.
Of a number of approaches proposed so far~\cite{8891809}, we introduce two inspiring studies.
%Both probabilistic and deterministic models are common in that they are driven by the same dynamics of the membrane potential.
%The difference is the way each neuron fires; the probabilistic model generates spikes in a stochastic manner, while the deterministic counterpart generates spikes when the membrane potential exceeds a certain threshold.
%We remark that one is not particularly better than the other, and we believe one idea developed for one model can be exported to the other.
%Thus, let us briefly review the literature of deterministic SNNs.

SpikeProp~\cite{Bohte2000} is one of the earliest attempts to develop a learning algorithm for deterministic SNNs.
SpikeProp uses backpropagation to minimize the difference between the target firing times $\{t_n^{\star}\}_{n=1}^N$ and the actual firing times $\{t_n\}_{n=1}^N$ of the network, \ie, $\sum_{n=1}^N |t_n^\star - t_n|^2$.
The gradient is approximated by assuming a linear relationship between the firing time and the membrane potential, which is valid only for a small learning rate.
%It is also difficult to handle creation or deletion of spikes, because the gradient in the threshold-based spike generation model does not tell it to us.

\citeauthor{NIPS2018_7417}~(\citeyear{NIPS2018_7417}) propose a differentiable alternative to the threshold-based spike generation, which facilitates gradient computation.
They employ a soft-threshold mechanism, and therefore, is differentiable without approximation.
%This allows us to compute the gradient without any approximation.
Another important contribution is that their model can handle not only spike trains but also a real-valued time-series.
They use a readout network that maps spike trains from/into a real-valued time-series.
This end-to-end formulation is significant towards practical applications of SNNs, and probabilistic SNNs should be equipped with this feature.
%While their approach is closely related to ours in that both resort to continuous relaxation for differentiability,

One interesting feature of probabilistic SNNs including our method is that both inference and learning algorithms can be executed naturally in an event-based manner
without any discretization of time axis.
This is in contrast to deterministic SNNs, where many learning algorithms require us to discretize the continuous-time dynamics for simulation.

\subsection{Differentiable Point Processes}
Our differentiable point process is significant in the community of point processes in that it largely expands the applicability of the reparameterization trick for point processes.
%although there are several existing approaches.
%our method significantly expands the applicability s, and especially, enables us to apply the reparameterization trick to SNNs.
%our method has a distinctive advantage over the others in that ours can make the existing point process differentiable.
Let us review the approaches to differentiable point processes, and discuss their pros and cons.

There are mainly three approaches to sample from point processes, and each of them can be used as a basis of differentiable point processes.
The first approach~\cite{Shchur2020} is to model the inter-event time conditioned on the past history by a log-normal mixture model, instead of modeling the conditional intensity function.
Since it is straightforward to develop a reparameterizable sampling algorithm for the mixture model, the resultant point process is also reparameterizable.
The second one is the inverse method~\cite{rasmussen2018}, which utilizes the fact that the inverse of the compensator $\Lambda^{[0,t]}$ can convert a unit-rate Poisson process into the point process with the corresponding conditional intensity function.
\citeauthor{NEURIPS2020_00ac8ed3}~(\citeyear{NEURIPS2020_00ac8ed3}) propose a reparameterization trick based on the inverse method.
The third one is the thinning algorithm, as we presented.

Of these three approaches, it is interesting to compare the second and the third approaches.
When applying the inverse method~\cite{NEURIPS2020_00ac8ed3} to computing ELBO, it is reported that the objective function contains discontinuous points, making optimization difficult.
The discontinuity arises because time stamps of a realization are parameterized, and the algorithm involves a discrete decision whether a time stamp is less than $T$ or not for termination.
In contrast, Our differentiable point process does not suffer from it because not time stamps but marks are parameterized.
In this sense, these two approaches are complementary.

When developing a path-wise gradient estimator for SNNs, only the third approach is feasible.
The first approach is difficult to be applied because SNNs are modeled via the conditional intensity function, and the inter-event time distribution is not available in a closed form.
The second approach is also difficult due to the lack of a closed-form expression of the inverse of the compensator.
Our approach only assumes the existence of an upperbound of the conditional intensity function, and therefore, can be applied to SNNs.
The assumption on the existence of a constant upperbound can be relaxed in the same way as Ogata's method~\cite{ogata1981}, which determines $\bar{\lambda}$ adaptively.

\section{Conclusion and Future Work}
We develop a path-wise gradient estimator for SNNs based on a differentiable point process.
Given the experimental results in Section~\ref{sec:empirical-study}, we conclude that our estimator has lower variance than the existing one, which contributes to improve the learning capability.

Throughout this paper, we only focus on the dependency of the gradient estimator on learning capability, and we have not discussed about its practical applications.
In the community of SNNs, however, an increasing number of studies have started to apply SNNs to real-world tasks~\cite{NEURIPS2018_82f2b308,Wozniak2020}.
One of the major concerns towards applying our method to real-world tasks is a method to convert real-valued data into/from spike trains.
While there are a number of information encoding methods for spike trains, it is still an open problem which encoding is preferred.
One interesting direction is to empirically and theoretically investigate the performance of different encoding methods and to understand their pros and cons.

Another limitation is the computational overhead as discussed in Section~\ref{sec:comp-time}.
While the probabilistic formulation can be simulated by an event-based manner, the gradient computation involves backpropagation through time~(BPTT), whose complexity increases proportionally to the number of spikes.
In addition to relying on the adaptive upperbound~$\bar{\lambda}$, applying online BPTT calculation and its approximation techniques~\cite{Williams1989} to SNNs may be an interesting research direction.

\newpage
\bibliography{prob_snn}
\bibliographystyle{icml2021}

\newpage
\appendix

\section{Relation to the Standard Notation}\label{sec:relat-stand-notat}
\begin{figure}[t]
\centering
 \includegraphics[width=.8\hsize]{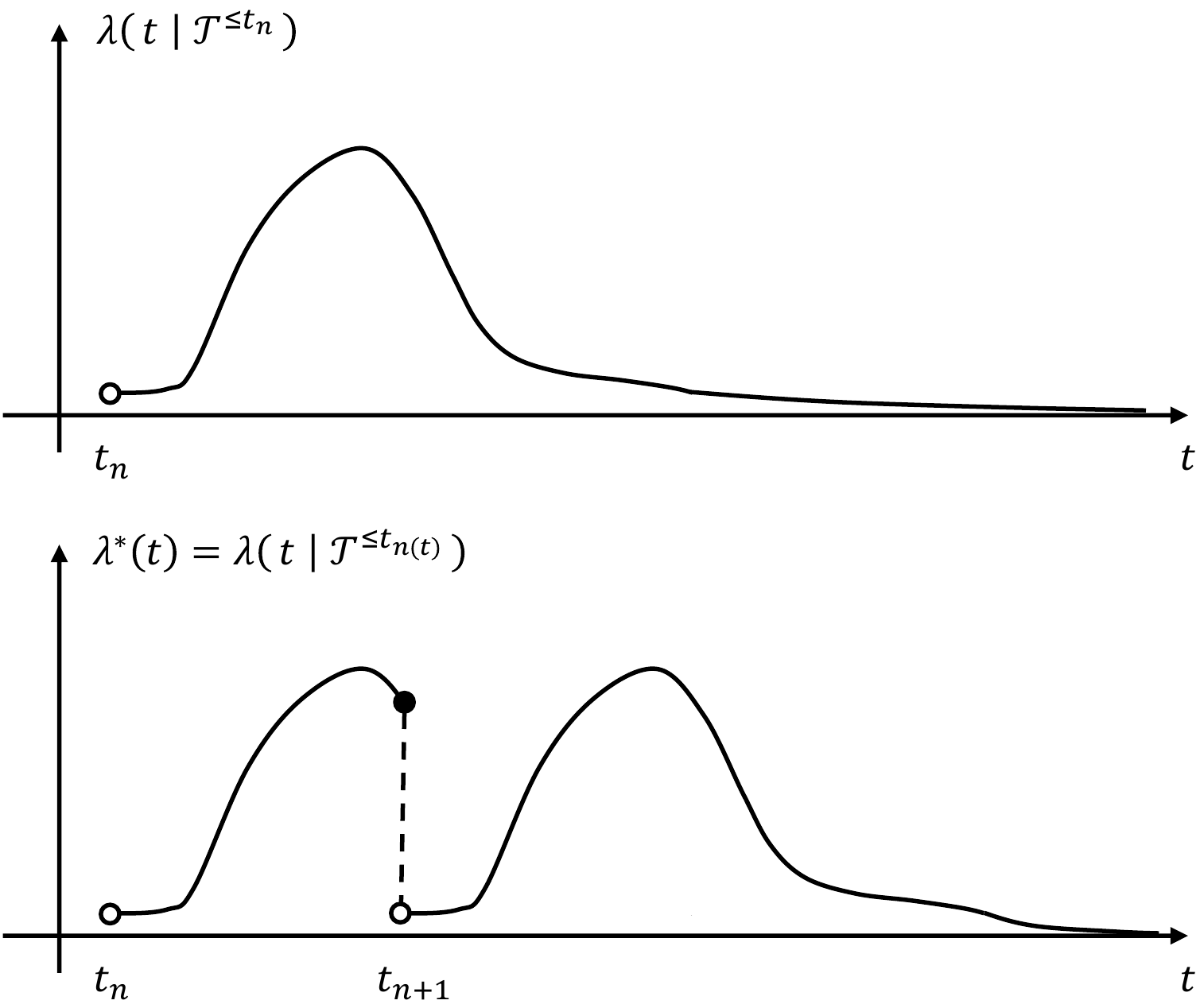}
 \caption{Conditional intensity functions of our definition~(top) and the standard one~(bottom). The standard conditional intensity function $\lambda^\ast$ can be obtained by joining multiple conditional intensity functions of our definition in a left-continuous way.}
 \label{fig:conditional-intensity}
\end{figure}
Our notation is different from the standard one employed by many others including a textbook~\cite{daley2003}.
A major difference is the conditional intensity function as illustrated in Figure~\ref{fig:conditional-intensity}.
Our conditional intensity function $\lambda(t\mid \cT^{\leq t_n})$ is defined for $t>t_n$ and is consistently conditioned on the history of events up to $t_n$~(the top panel of Figure~\ref{fig:conditional-intensity}).
On the other hand, the standard conditional intensity function $\lambda^\ast(t)$ is defined for all $t$, and the history that conditions it is ambiguous and depends on the context;
it is sometimes conditioned on the history of events that occured before~(and not including) $t$, which is represented as $\cT^{\leq t_{n(t)}}$~(the bottom panel of Figure~\ref{fig:conditional-intensity}), and it is sometimes equivalent to our definition of the conditional intensity function.

While such an ambiguity helps to simplify equations~(\eg, the compensator $\Lambda^{[0,T]}=\int_0^T \lambda^\ast(t)\dl{t}$ can be represented by a single integral),
it is sometimes very confusing especially for those who are not familiar with temporal point processes. Therefore, we employ a less ambiguous definition.
In order to represent the standard conditional intensity function by our conditional intensity function, we introduced the left-continuous counting process $n(t)$ as illustrated in Figure~\ref{fig:counting}; with this counting process, we obtain the relationship between the standard and our conditional intensity functions, $\lambda^\ast(t)=\lambda(t\mid\cT^{\leq t_{n(t)}})$.

\section{Conditions for Conditional Intensity Function}\label{sec:proof-mpp-unique}
This section provides a proof of Proposition~\ref{prop:mult-point-proc-unique}, which states several conditions under which the conditional intensity function can specify a marked point process uniquely.

\begin{proof}[Proof of Proposition~\ref{prop:mult-point-proc-unique}]
 Since Equation~\eqref{eq:8} reads as,
\begin{align*}
 \int_X \dl{\bp} \lambda(t,\bp \mid \cT_X^{\leq t_n})  &= \dfrac{f\left(t \mid \cT_{X}^{\leq t_{n}}\right)}{1 - F\left(t \mid \cT_{X}^{\leq t_{n}}\right)}\\
 &= - \diff*{\log(1 - F(t \mid \cT_{X}^{\leq t_n}))}{t},
\end{align*}
 we can represent the cumulative distribution function by the conditional intensity function as,
 \begin{align}
  \label{eq:22}F(t \mid \cT_X^{\leq t_n}) = 1 - \exp\left({-\int_{t_n}^t \dl{s} \int_{X}\dl{\bp} \lambda(s,\bp \mid \cT_{X}^{\leq t_n})}\right).
 \end{align}
 In order for the conditional intensity function to define a proper cumulative distribution function, the function $F(t \mid \cT_{X}^{\leq t_n})$ defined as Equation~\eqref{eq:22} must satisfy the following four conditions:
 \begin{enumerate}
  \item $\displaystyle\lim_{t\rightarrow\infty} F(t \mid \cT_{X}^{\leq t_n}) = 1$,
  \item $\displaystyle\lim_{t\rightarrow -\infty} F(t \mid \cT_{X}^{\leq t_n}) = 0$,
  \item $F(t \mid \cT_{X}^{\leq t_n})$ is non-decreasing in $t$, and 
  \item $F(t \mid \cT_{X}^{\leq t_n})$ is right-continuous.
 \end{enumerate}
 
 Condition~1 is satisfied by Assumption~2.
 Condition~2 is satisfied because $F(t \mid \cT_X^{\leq t_n})=0$ holds for any $t\leq t_n$.
 Condition~3 holds because $\lambda(t, \bp)\geq 0$ implies that the integral in the exponential function in Equation~\eqref{eq:22} is non-decreasing.
 Condition~4 holds because Assumption~3 ensures that the exponent in Equation~\eqref{eq:22} is right-continuous, which indicates that Equation~\eqref{eq:22} itself is right-continuous.
\end{proof}

\section{Properties of Differentiable Point Processes}\label{appendix:prop-diff-point}
This section provides several properties of differentiable point processes and their proofs.
First, let us prove Proposition~\ref{prop:diff2mark}, which clarifies the conditional intensity function of a differentiable point process.

\begin{proof}[Proof of Proposition~\ref{prop:diff2mark}]
The following calculus clarifies the relationship:
\begin{align*}
\nonumber & \Pr\left[t_{n+1}\in[t, t+\dl{t}], \bp_{n+1}=\bp \mid \cT_{\convzero{\bbone^D}}^{<t}\right]\\
\nonumber = & \Pr\left[t_{n+1}\in[t, t+\dl{t}]\mid \cT_{\convzero{\bbone^D}}^{<t}\right] \\
\nonumber &\cdot p\left(\bp_{n+1}=\bp \mid t_{n+1}\in[t, t+\dl{t}], \cT_{\convzero{\bbone^D}}^{< t}\right)\\
 = & \bar{\lambda} \dl{t} \cdot g_\tau\left(\begin{bmatrix} \bp\\ 1- \|\bp\|_1 \end{bmatrix}; \bpi_{\bar{\lambda}}\circ\bm{\lambda}\left(t \mid \cT_{\convzero{\bbone^D}}^{\leq t_n}\right)\right),
\end{align*}
 where let $\cT_{\convzero{\bbone^D}}^{<t}$ denote the event $t_{n+1} \notin (t_n, t)$ and $\cT_{\convzero{\bbone^D}}^{\leq t_n}$.
 This suggests that in a differentiable point process, time stamps are distributed according to the homogeneous Poisson process with intensity $\bar{\lambda}$, and each mark is distributed according to the concrete distribution.
\end{proof}

We then investigate two properties of the differentiable point process.
Proposition~\ref{prop:differentiability} states that a realization of the differentiable point process is differentiable with respect to model parameters under mild conditions.
Proposition~\ref{prop:convergence} states that the differentiable point process becomes equivalent to the original point process as temperature goes zero.

\begin{prop}%[Differentiability]
\label{prop:differentiability}
 Let $\lambda_\theta(t, \bp \mid \cT_{\bbone^D}^{\leq t_n})$ be a conditional intensity function of a multivariate point process parameterized by $\theta$.
 Assume that the conditional intensity function can be calculated with $\cT_{\convzero{\bbone^D}}^{\leq t_n}$ and is differentiable with respect to any mark in $\cT_{\convzero{\bbone^D}}^{\leq t_n}$ and $\theta$.
 %can be defined for $\{\bp_n\in\convzero{\bbone^D}\}_{n=1}^{n(t)}$ and 
 %is differentiable with respect to $\{\bp_n\in\convzero{\bbone^D}\}_{n=1}^{n(t)}$ and $\theta$.
 Then, the marks of a realization of the corresponding differentiable point process is differentiable with respect to $\theta$.
\end{prop}

\begin{proof}
We prove Proposition~\ref{prop:differentiability} by induction.
 The first mark is distributed according to $\bm{\pi}_{\bar{\lambda}}\circ\bm{\lambda}_\theta\left(t \mid \emptyset\right)$, which is differentiable with respect to $\theta$.
 Assume that marks observed up to (but not including) time $t$ are differentiable with respect to $\theta$.
 A mark $\bp$ at time $t$ is a realization of the concrete distribution with parameter $\bm{\pi}:=\bm{\pi}_{\bar{\lambda}}\circ\bm{\lambda}_\theta\left(t \mid \cT_{\convzero{\bbone^D}}^{\leq t_n}\right)$, and thus, is differentiable with respect to $\bm{\pi}$.
 $\bm{\pi}$ is differentiable with respect to $\bm{\lambda}_\theta$, and $\bm{\lambda}_\theta$ is assumed to be differentiable with respect to $\theta$ and the past marks, which are differentiable by assumption, and this completes the proof.
\end{proof}

\begin{prop}
 \label{prop:convergence}
 Assume that $\lambda\left(t, \bp \mid \cT_{\convzero{\bbone^D}}^{\leq t_n}\right)=\lambda\left(t, \bp \mid \cT_{\convzero{\bbone^D}}^{\leq t_n}\backslash\{(t_k, \bp_k)\}\right)$ holds for any $\cT_{\convzero{\bbone^D}}^{\leq t_n}$ and any $k\in[n]$ such that $\bp_k=\mathbf{0}$.
 Then, in the limit of $\tau\rightarrow +0$, the output of Algorithm~\ref{alg:mult-diff-point} is distributed according to $\mathcal{MPP}(\lambda)$ if we discard the event with mark $\bp=\mathbf{0}$.
\end{prop}

\begin{proof}
 As proven by \citeauthor{maddison2017}~\citeyear{maddison2017}, the random variable following the concrete distribution converges to the one-hot representation of the categorical variate in the limit of $
\tau\rightarrow +0$.
The random variable in line~6 of Algorithm~\ref{alg:mult-diff-point} satisfies the following,
\begin{align}
 \Pr\left[\lim_{\tau\rightarrow +0} \begin{bmatrix}\bp \\ r\end{bmatrix}=\bone_d\mid s, \cT\right] = \dfrac{\lambda(s, \bone_d \mid \cT)}{\bar{\lambda}},
\end{align}
 for $d\in[D+1]$, 
 where let $\lambda(s, \bone_{D+1} \mid \cT) \equiv \bar{\lambda} - \sum_{d=1}^{D}\lambda(s,\bone_d\mid \cT)$.
 The above expression states that if the value of the conditional intensity function is the same,
 the random variable in line~6 of Algorithm~\ref{alg:mult-diff-point} is equivalent to that in line~6 of Algorithm~\ref{alg:thinning-multivariate}.
 The only difference between these algorithms is whether the event with zero mark $\bp=\mathbf{0}$ is discarded~(Algorithm~\ref{alg:thinning-multivariate}) or not~(Algorithm~\ref{alg:mult-diff-point}).
 If the assumption made in Proposition~\ref{prop:convergence} is satisfied, zero marks do not affect the conditional intensity function, and thus, the output of Algorithm~\ref{alg:mult-diff-point} is distributed according to $\mathcal{MPP}(\lambda)$ if we discard the events with zero marks.
\end{proof}

\section{Properties of $\partial$SNNs}
\label{appendix:prop-diff-spik}
This section describes the properties of $\partial$SNNs.
First, Proposition~\ref{prop:uniqueness} states that the conditional intensity function~(Equation~\eqref{eq:17}) satisfies all of the conditions listed in Proposition~\ref{prop:mult-point-proc-unique}, and thus, it defines an $\bar{\cN}$-marked point process uniquely.
\begin{prop}
 \label{prop:uniqueness}
 Assume that the filter functions $\{f_{d^\prime,d}(s)\}_{d,d^\prime\in[D]}$ are continuous with respect to $s$.
 Then, the conditional intensity function~(Equation~\eqref{eq:17}) uniquely defines an $\bar{\cN}$-marked point process.
\end{prop}

\begin{proof}
 We will confirm the assumptions of Proposition~\ref{prop:mult-point-proc-unique}.
 Observing that,
 \begin{align*}
  &\int_{t_n}^{t}\dl{s} \int_{\bar{\cN}} \dl{\bp} \lambda\left(s, \bp \mid \cT_{\bar{\cN}}^{\leq t_n}\right) \\
  = &
  \int_{t_n}^{t}\dl{s} \left[\int_{\convzero{\cH}} \dl{\bp} \lambda\left(s, \bp \mid \cT_{\bar{\cN}}^{\leq t_n}\right) \right.\\
  & \hspace{1.5cm}\left.+ \sum_{\bp\in\cO} \lambda\left(s, \bp \mid \cT_{\bar{\cN}}^{\leq t_n}\right) \right]\\
  = & \int_{t_n}^{t}\dl{s} \left[\bar{\lambda} + \sum_{\bp\in\cO} \lambda\left(s, \bp \mid \cT_{\bar{\cN}}^{\leq t_n}\right) \right]\\
  = & \bar{\lambda}(t-t_n) + \int_{t_n}^{t}\dl{s} \sum_{\bp\in\cO} \lambda\left(s, \bp \mid \cT_{\bar{\cN}}^{\leq t_n}\right),
 \end{align*}
 the first condition is satisfied.
 By taking $t\rightarrow\infty$ in the above expression, we can confirm that the second condition is satisfied~(the first term $\bar{\lambda}(t-t_n)$ goes to infinity
 and the second term is guaranteed to be non-negative).
 The third condition is satisfied because $\lambda(s,\bp \mid \cT_{\bar{\cN}}^{\leq t_n})$ is a continuous function with respect to $s>t_n$, which can be guaranteed by the continuity of the filter functions.
\end{proof}

Then, let us discuss two properties of $\partial$SNN.
Proposition~\ref{prop:differentiable-log-likelihood} states that the objective function is differentiable with respect to $\phi$.
Proposition~\ref{prop:converge-snn} states that $\partial$SNN  becomes equivalent to the vanilla SNN in the limit of $\tau\rightarrow +0$.
\begin{prop}
 \label{prop:differentiable-log-likelihood}
 Assume that the filter functions $\{f_{d^\prime,d}(s)\}_{d,d^\prime\in[D]}$ are differentiable with respect to their parameters $\phi$.
 The Monte-Carlo approximation of ELBO~(Equation~\eqref{eq:14}) is differentiable with respect to $\phi$ if we employ $\mathcal{\partial PP}(\lambda_q(t, \bp \mid \cT_{\bar{\cN}};\phi); \bar{\lambda}, \tau)$ as the variational distribution.
\end{prop}

\begin{proof}
 We first show that marks of a realization of the variational distribution $\cT_{\cH}(\phi)$ are differentiable with respect to $\phi$.
 We then show that both $\log p(\cT_{\cO}, \cT_{\convzero{\cH}})$ and $\log q(\cT_{\convzero{\cH}})$ are differentiable with respect to the marks of $\cT_{\cH}(\phi)$.
 We finally show that $\log q(\cT_{\convzero{\cH}};\phi)$ is differentiable with respect to $\phi$.
 By the fact that the composition of differentiable functions is differentiable, the proposition is implied by these three statements.

 $\cT_{\cH}(\phi)$ is sampled by Algorithm~\ref{alg:mult-diff-point} using the conditional intensity function $\lambda_q(t,\bp \mid \cT_{\bar{\cN}};\phi)$, which is differentiable with respect to any mark in $\cT_{\bar{\cN}}$ and $\phi$~(by the assumption).
 Therefore, by Proposition~\ref{prop:differentiability}, the marks of $\cT_{\cH}(\phi)$ are differentiable with respect to $\phi$.

The logarithm of the joint distribution $\log p(\cT_{\cO},\cT_{\convzero{\cH}})$ can be written as,
 \begin{align*}
  &\log p(\cT_{\cO},\cT_{\convzero{\cH}}; \theta) \\
= &\sum_{(t,\bp)\in\cT_{\cO}} \log \lambda^{\mathrm{SNN}}\left(t, \bp \mid \cT_{\bar{\cN}}^{\leq t_{n(t)}}\right) \\
  & + \sum_{(t,\bp)\in\cT_{\convzero{\cH}}} \log \lambda_{\partial}\left(t, \bp_{\cH} \mid \cT_{\bar{\cN}}^{\leq t_{n(t)}};\blambda_\cH, \bar{\lambda}, \tau\right)\\
  & - \bar{\lambda} T - \int_{0}^T \dl{s}\sum_{\bp\in\cO}\lambda^{\mathrm{SNN}}\left(s,\bp \mid \cT_{\bar{\cN}}^{\leq t_{n(s)}}\right).
 \end{align*}
 The first and the last terms are differentiable with respect to marks of $\cT_{\convzero{\cH}}$, and the third term does not depend on $\cT_{\convzero{\cH}}$.
 For any $(t,\bp)\in\cT_{\convzero{\cH}}$, the summand of the second term,
 \begin{align*}
  & \log \lambda_{\partial}\left(t, \bp_{\cH} \mid \cT_{\bar{\cN}}^{\leq t_{n(t)}};\blambda_\cH, \bar{\lambda}, \tau\right)\\
  = & \log\bar{\lambda} + \log g_\tau\left(
\begin{bmatrix} 
\bp_{\cH}\\ 1- \|\bp_{\cH}\|_1 
\end{bmatrix}; 
\bpi_{\bar{\lambda}}\circ\bm{\lambda}_{\cH}\left(t \mid \cT_{\bar{\cN}}^{\leq t_n(t)}\right)\right),
 \end{align*}
 is differentiable with respect to $\bp_\cH$ because the probability density function of the concrete distribution is differentiable with respect to $\bp_\cH$.
 It is also differentiable with respect to the past marks in $\cT_{\bar{\cN}}^{\leq t_{n(t)}}$ because the probability density function $g_\tau$ is differentiable with respect to its parameter $\bm{\pi}:=\bm{\pi}_{\bar{\lambda}}\circ\bm{\lambda}_\cH\left(t \mid \cT_{\bar{\cN}}^{\leq t_n(t)}\right)$, which is differentiable with respect to the marks in $\cT_{\bar{\cN}}^{\leq t_n(t)}$.
 Therefore, $\log p(\cT_\cO, \cT_{\convzero{\cH}};\theta)$, is differentiable with respect to the marks of $\cT_{\convzero{\cH}}$.

 The logarithm of the variational distribution $\log q(\cT_{\convzero{\cH}})$ can be written as,
 \begin{align}
\nonumber  &\log q(\cT_{\convzero{\cH}}) \\
\label{eq:2} =& \sum_{(t,\bp)\in\cT_{\convzero{\cH}}} \log \lambda_{\partial}\left(t, \bp_{\cH} \mid \cT_{\bar{\cN}}^{\leq t_{n(t)}};\blambda_\cH, \bar{\lambda}, \tau\right)
- \bar{\lambda}T,
 \end{align}
 which is also differentiable with respect to the marks of $\cT_{\convzero{\cH}}$ in the same way as the above discussion.

 Finally, $\log q(\cT_{\convzero{\cH}};\phi)$ is differentiable with respect to $\phi$, because the first term of Equation~\eqref{eq:2} is differentiable with respect to $\blambda_{\cH}$, which is differentiable with respect to $\phi$~(partially by the assumption that the filter functions are differentiable with respect to $\phi$).

The proposition follows from the three statements above.
\end{proof}

\begin{prop}
 \label{prop:converge-snn}
 In the limit of $\tau\rightarrow +0$,
 a realization of $\partial$SNN~(Equation~\eqref{eq:17}) is distributed according to
 the vanilla SNN~(Equation~\eqref{eq:10}) if we discard events with mark $\bp=\mathbf{0}$.
\end{prop}

\begin{proof}
 Since for any event $(t_k,\bp_k)$~($k\in[n]$) such that $\bp_k=\mathbf{0}$,
\begin{align*}
 \lambda^{\mathrm{SNN}}\left(t,\bp \mid \cT_{\bar\cN}^{\leq t_n}\right) = \lambda^{\mathrm{SNN}}\left(t,\bp \mid \cT_{\bar\cN}^{\leq t_n}\backslash\{(t_k,\bp_k)\}\right),
\end{align*}
 holds, the event with mark $\bp=\mathbf{0}$ has no influence on the conditional intensity function.
 In the limit of $\tau\rightarrow +0$,
 \begin{align*}
  & {\lambda^{\partial\mathrm{SNN}}\left(t,\bp \mid \cT_{\bar\cN}^{\leq t_n}\right)}\\
  = &
  \begin{cases}
   \lambda^{\mathrm{SNN}}\left(t,\bp \mid \cT_{\bar\cN}^{\leq t_n}\right) & (\bp\in\cO\cup\cH),\\
   %1-\displaystyle\sum_{\bp\in\cO\cup\cH} \lambda^{\mathrm{SNN}}\left(t_{n+1},\bp \mid \cT_{\bar\cN}^{\leq t_n}\right) \dl{t} & (\bp=\mathbf{0}),\\
   \displaystyle\left(\bar{\lambda} - \sum_{\bp\in\cH} \lambda^{\mathrm{SNN}}\left(t,\bp \mid \cT_{\bar\cN}^{\leq t_n}\right)\right) & (\bp=\mathbf{0}),
  \end{cases}
 \end{align*}
 holds.
 As discussed above, the event with mark $\bp=\mathbf{0}$ has no influence on computing the conditional intensity function, and can be removed without changing the conditional intensity function.
 Therefore, a realization of $\mpp(\lambda^{\partial\mathrm{SNN}})$ is equivalent to that of $\mpp(\lambda^{\mathrm{SNN}})$ if we discard all the events with mark $\bp=\mathbf{0}$.
\end{proof}

\section{Numerically Stable Implementation}
\label{sec:numer-stable-impl}
For numerical stability, we recommend to represent all probability and subprobability vectors and conditional intensity functions in the logarithmic scale.
In this section, we describe an accurate computation of the parameter of the concrete distribution.

When computing the parameter of the concrete distribution shown below in the logarithmic scale,
\begin{align*}
 \begin{split}
 &\bpi_{\bar{\lambda}}\circ\bm{\lambda}\left(t \mid \cT_{\convzero{\bbone^D}}^{\leq t_n}\right)\\
 = &\frac{1}{\bar{\lambda}}
 \begin{bmatrix}
 \bm{\lambda}\left(t \mid \cT_{\convzero{\bbone^D}}^{\leq t_n}\right) & \bar{\lambda} - \left\|\bm{\lambda}\left(t \mid \cT_{\convzero{\bbone^D}}^{\leq t_n}\right)\right\|_1
 \end{bmatrix},
\end{split}
\end{align*}
it is straightforward to compute the first $D$ elements with the logarithm of the conditional intensity function, $\log\lambda\left(t, \bp \mid \cT_{\convzero{\bbone^D}}^{<t}\right)$.
However, the last element,
\begin{align}
 \label{eq:3}
1-\frac{\sum_{\bp^\prime\in \bbone^D}\lambda\left(t, \bp^\prime \mid \cT_{\convzero{\bbone^D}}^{<t}\right)}{\bar{\lambda}}, 
\end{align}
 is not trivial to compute accurately in the logarithmic scale.

We resort to $\mathtt{log1mexp}$~\cite{Machler2012} to compute it,
which allows us to compute $\log(1 - \exp(-|a|))$ accurately for $a\neq 0$ as follows:
\begin{align*}
 \mathtt{log1mexp}(a) =
 \begin{cases}
  \log(-\mathtt{expm1}(-a)) & (0<a\leq \log 2),\\
  \mathtt{log1p}(-\exp(-a)) & (a> \log 2),
 \end{cases}
\end{align*}
where $\mathtt{expm1(x)}$ and $\mathtt{log1p}(x)$ approximately compute $\exp(x)-1$ and $\log(1+x)$ respectively by using a few terms of their Taylor series.
Since we can compute $\log p := \log\left[{\sum_{\bp^\prime\in \bbone^D}\lambda\left(t, \bp^\prime \mid \cT_{\convzero{\bbone^D}}^{<t}\right)}\right] - \log{\bar{\lambda}}$ by $\mathtt{logsumexp}$,
the computation of Equation~\eqref{eq:3} boils down to the computation of $\log (1-p)$ given $\log p$~($p\in[0,1)$).
Since $\log(1-p) = \log(1-\exp(-|\log p|))$ holds for $p\in[0,1)$, we can utilize $\mathtt{log1mexp}$ to compute it.
\end{document}